\documentclass{article} %
\usepackage{iclr2027_conference,times}

\usepackage{amsmath,amsfonts,bm}

\def\eqref#1{equation~\ref{#1}}
\def\Eqref#1{Equation~\ref{#1}}

\def\1{\bm{1}}

\def\vf{{\bm{f}}}

\def\vu{{\bm{u}}}
\def\vv{{\bm{v}}}

\def\vx{{\bm{x}}}

\DeclareMathAlphabet{\mathsfit}{\encodingdefault}{\sfdefault}{m}{sl}
\SetMathAlphabet{\mathsfit}{bold}{\encodingdefault}{\sfdefault}{bx}{n}

\newcommand{\KL}{D_{\mathrm{KL}}}

\usepackage{hyperref}
\usepackage{url}
\usepackage{booktabs}
\usepackage{multirow}
\usepackage{algorithm}
\usepackage{algpseudocode}
\usepackage{amsthm, amssymb}
\usepackage{thmtools,thm-restate}
\usepackage{colortbl}
\usepackage{placeins}
\usepackage{graphicx}
\usepackage{subcaption}
\usepackage{xcolor}
\usepackage{wrapfig}
\usepackage{soul}
\usepackage{longtable}
\usepackage{makecell}
\usepackage{array}
\usepackage{float}
\usepackage{tabularx}
\usepackage{xfp}

\def\eqref#1{Eq.~(\ref{#1})}
\def\Eqref#1{Eq.~(\ref{#1})}
\hypersetup{colorlinks=true, citecolor=black, linkcolor=black, urlcolor=black}
\definecolor{lavendertext}{HTML}{8A3FFC}   
\definecolor{lavender}{HTML}{E6E6FA}
\newsavebox{\ablA}\newsavebox{\ablB}\newsavebox{\ablC}

\algrenewcommand{\algorithmiccomment}[1]{\hfill\textcolor{gray}{$\triangleright$ \textit{#1}}}

\title{FastOPD: On-Policy Distillation for \\ Lightweight VLA Deployment}

\iclrfinalcopy 

\author{%
Yoojin Oh$^{1*}$, Jeongsol Kim$^{1*}$, Yeonwoo Seo$^{1}$, Jangho Park$^{1}$, Seonghyun Jin$^{1}$, \\
\textbf{Sunwoo Park$^{1}$, Youngmin Kim$^{1}$, Youngjun Jun$^{1}$, Kyumin Choi$^{2}$, Jong Chul Ye$^{1}$} \\
$^{1}$KAIST \quad $^{2}$Sungkyunkwan University
}

\begin{document}

\maketitle
\lhead{Preprint.}
{\renewcommand{\thefootnote}{*}\footnotetext{Equal contribution.\hfill Project page is available \href{https://fastopd.github.io/}{\textcolor{lavendertext}{here}}.}}

\begin{abstract}
Vision-Language-Action (VLA) foundation models have scaled rapidly to enhance manipulation performance and generalizability, but this scaling incurs high computational costs that render real-world deployment increasingly challenging.
Existing approaches typically mitigate this issue by designing smaller architectures or reducing the iterative denoising steps in flow-based policies.
In this work, we propose \textit{FastOPD}, a foundation-to-lightweight VLA framework that enables the practical deployment of large-scale VLAs through efficient on-policy distillation.
Specifically, FastOPD adapts a flow map for single-state teacher supervision and combines it with a self-consistency objective to construct a compact student that learns the teacher dynamics.
Furthermore, we theoretically demonstrate that minimizing this objective allows the distilled student to recover a distribution on par with that induced by an ideal few-step teacher model.
We evaluate FastOPD across diverse foundation policies in simulation and real-world experiments.
On LIBERO, FastOPD retains 84\% of the performance of $\pi_{0.5}$ with only two inference steps, reducing inference latency by 78.1\% while outperforming existing few-step distillation baselines in average success rate. With LingBot-VLA as the teacher, FastOPD improves the single-step success rate over the base student by 15.9 percentage points on RoboTwin~2.0. We further demonstrate its applicability to a World Action Model (WAM) and deploy a compact student distilled from MolmoAct2 on a real robot.

\end{abstract}
\begin{figure}[h]
    \centering
    \includegraphics[width=0.84\linewidth]{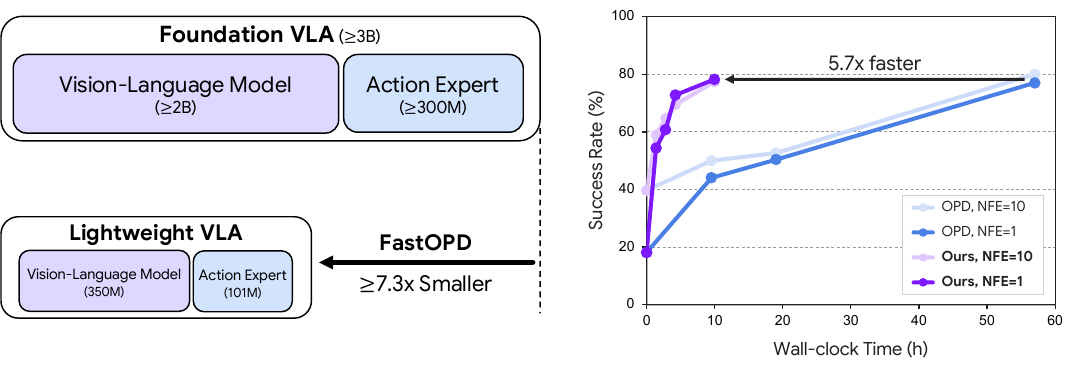}
    \vspace{-3.5mm}
    \caption{
    \textbf{Foundation-to-lightweight VLA policy distillation with FastOPD.} (Left) Our method enables a small VLA to inherit large-scale VLA knowledge. (Right) FastOPD reaches comparable performance to conventional on-policy distillation~\citep{li2026diffusionopd} 5.7$\times$ faster.
     }
     
\label{fig:figure1}
\end{figure}

\section{Introduction}
Recent progress in robot learning has moved toward large-scale foundation policies, ranging from Vision-Language-Action (VLA) models~\citep{o2024open,black2024pi_0,intelligence2025pi_,bjorck2025gr00t,wu2026pragmatic,fang2026molmoact2actionreasoningmodels} to emerging World Action Models (WAMs)~\citep{kim2026cosmos,ye2026world,yuan2026fast}. In particular, recent methods such as DreamZero~\citep{ye2026world} and Cosmos Policy~\citep{kim2026cosmos} leverage pretrained video generation models to capture rich visual dynamics and achieve strong generalization and manipulation performance.
However, these powerful foundation priors also lead to increasingly large policy backbones and high computational cost, creating a tension between model capability and deployment efficiency on physical robots.
Lightweight VLA models such as SmolVLA~\citep{shukor2025smolvla} address the efficiency side of this trade-off by directly designing smaller policies with less than 0.5B parameters.

A complementary direction is to transfer the capability of a powerful foundation policy into a compact student through knowledge distillation.
Knowledge Distillation (KD)~\citep{hinton2015distilling, mansourian2025comprehensive} has been widely studied for transferring the capability of a large teacher model into a smaller student. Specifically, On-Policy Distillation (OPD)~\citep{agarwal2024policy} supervises the student using teacher predictions on its own generated trajectories, reducing distribution mismatch and compounding errors from offline distillation. OPD also provides dense supervision along the student trajectory, in contrast to the sparse rewards commonly used in reinforcement learning. While OPD has been actively explored in LLM post-training, recent works have extended it to diffusion and flow models~\citep{li2026diffusionopd,fang2026flow}.
However, directly applying OPD to flow-based policies is computationally inefficient. It requires rolling out the student's denoising trajectory and repeatedly evaluating the teacher at the visited timesteps, resulting in high training cost for large foundation teachers.

In this work, we revisit few-step flow-map models through the lens of
on-policy distillation (OPD) and introduce \textbf{FastOPD}, a
foundation-to-lightweight VLA framework for efficient policy distillation.
Our goal is to transfer the capabilities of a large-scale VLA into a
compact student while substantially reducing both training and inference
costs, without sacrificing policy performance
(see Fig.~\ref{fig:figure1}).
To reduce the training cost of conventional OPD, FastOPD combines
on-policy teacher supervision of the student's velocity field with a
self-consistency objective~\citep{frans2025one}.
Rather than querying teacher supervision throughout an entire
generation trajectory, the self-consistency objective propagates local
velocity supervision to finite-interval shortcut transitions, as
illustrated in Fig.~\ref{fig:conceptual-figure}.
This allows the student to approximate the teacher dynamics without
explicitly constructing or evaluating a teacher flow map over each
shortcut transition. 
We further show theoretically that FastOPD implicitly recovers the
distribution induced by the teacher's ideal shortcut model, providing a
distributional justification for few-step distillation. 

\begin{figure}[t]
    \centering
    \begin{subfigure}{0.465\linewidth}
        \centering
        \includegraphics[width=\linewidth]{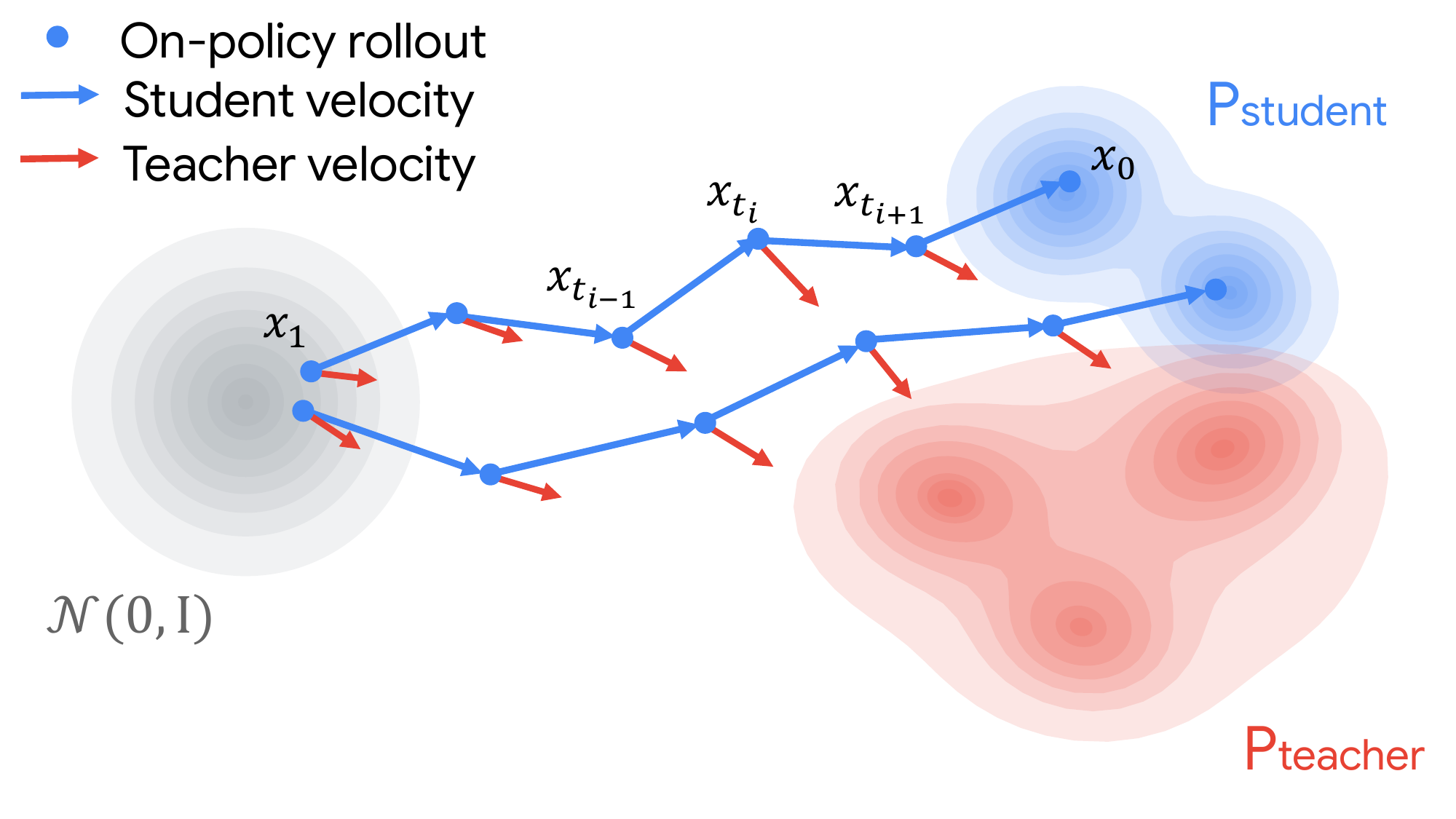}
        \caption{DiffusionOPD}
        \label{fig:concept1}
    \end{subfigure}
    \begin{subfigure}{0.465\linewidth}
        \centering
        \includegraphics[width=\linewidth]{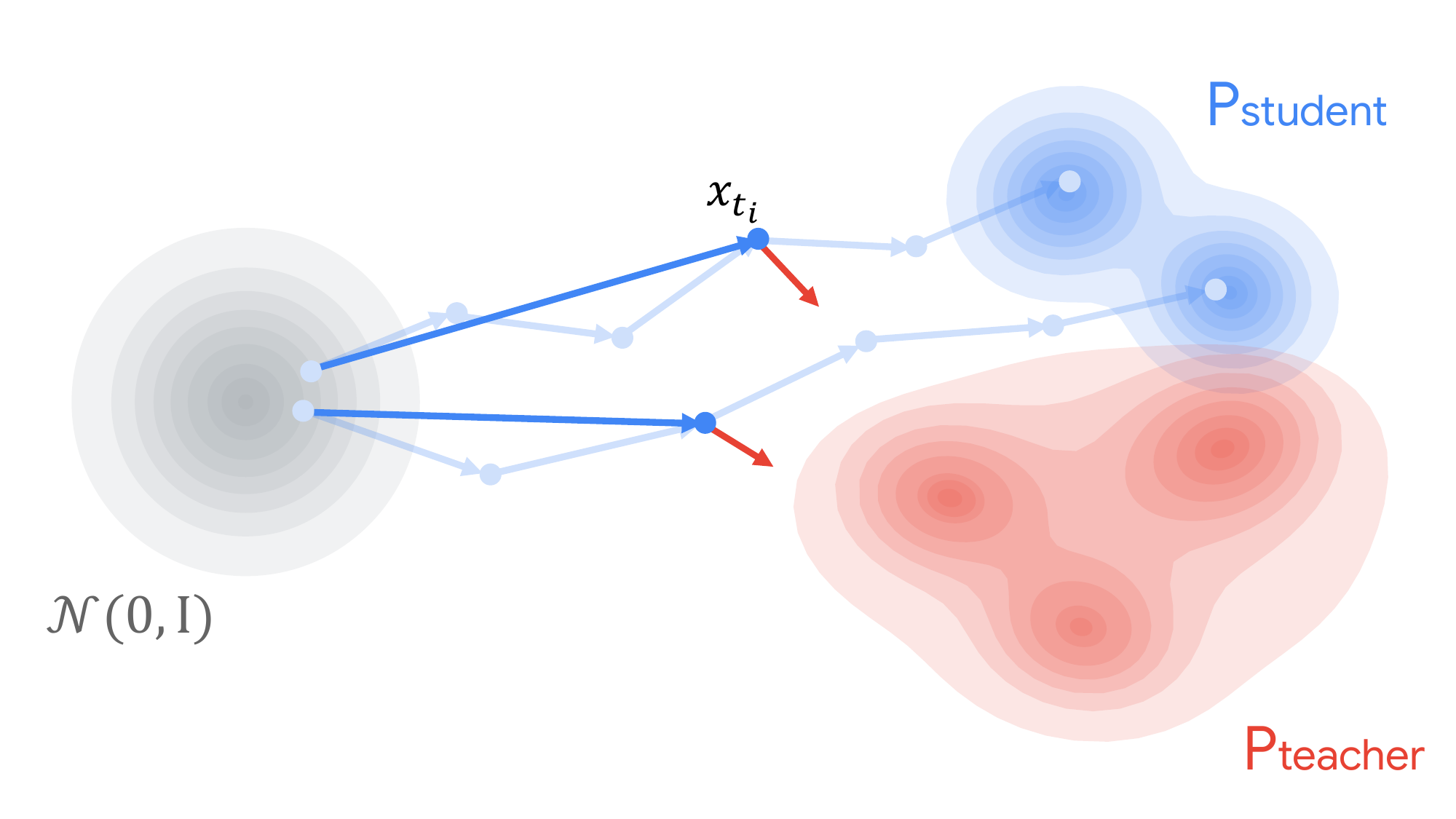}
        \caption{FastOPD (Ours)}
        \label{fig:concept2}
    \end{subfigure}
    \caption{\textbf{Comparison between DiffusionOPD and FastOPD.}
(a) DiffusionOPD evaluates  teacher at every denoising step.
(b) FastOPD evaluates  teacher at a single sampled state reached by a flow-map jump from noise.
Blue arrows denote student transitions; red arrows denote teacher velocities.}
    \label{fig:conceptual-figure}
    \vspace{-7mm}
\end{figure}

Importantly, the same self-consistency formulation naturally induces an
any-step flow map~\citep{boffi2024flow,song2023consistency,kim2024consistency,geng2025mean,sabour2025align},
allowing the distilled student to operate with substantially fewer
sampling steps at inference.
FastOPD therefore compresses the foundation VLA along two complementary
dimensions: model capacity and sampling complexity, enabling deployment
on resource-constrained, consumer-grade hardware.
The proposed formulation also enables stable training in the challenging
few-step regime without adversarial objectives or costly higher-order
derivative computations.
A concurrent work, VLA-OPD~\citep{zhong2026vla}, also studies on-policy
distillation for VLA models, but primarily uses expert supervision to
improve policy performance.
In contrast, FastOPD focuses on transferring a foundation policy into a
lightweight few-step student.

We evaluate FastOPD on two simulation benchmarks and one real-world
manipulation task, using diverse foundation teachers with 3B--6B
parameters~\citep{intelligence2025pi_,yuan2026fast,fang2026molmoact2actionreasoningmodels}.
To maximize both architectural and inference efficiency, we adopt
SmolVLA as the lightweight student policy.
We compare against representative few-step distillation methods,
including CTM, DMD, and iMF.
Across these settings, FastOPD enables the compact student to recover up
to 84\% of the performance of teacher models that are more than
$7.3\times$ larger, using only one or two sampling steps.
This substantially reduces inference cost while retaining a large
fraction of the foundation policy's performance.
To the best of our knowledge, FastOPD is the first OPD framework for
VLAs that jointly reduces both model size and the number of sampling
steps while retaining the capabilities of a large-scale foundation
policy.

\section{Background} 
\textbf{Flow-based Models.} Diffusion and flow-based models~\citep{song2021scorebased,ho2020denoising,esser2024scaling,lipman2023flow} generate high-quality samples through iterative denoising. %
Let $p_1=\mathcal N(\mathbf 0,\mathbf I)$ be a Gaussian prior and $p_0$ the context-conditioned data distribution over actions. 
Flow matching~\citep{lipman2023flow, esser2024scaling} defines a stochastic interpolant $\vx_t=(1-t)\vx_0+t\vx_1$, with $\vx_0\sim p_0$, $\vx_1\sim p_1,\ t\in[0,1]$, and learns the marginal velocity field $\vv(\vx,t):=\mathbb E[\vx_1-\vx_0\mid \vx_t=\vx]$ via%
\begin{equation}
    \mathbb E_{t,\vx_0,\vx_1}\big[\|\vv_\theta(\vx_t,t) - (\vx_1-\vx_0)\|_2^2\big].
\end{equation}
Sampling integrates the probability-flow ODE $\dot{\vx}_t=\vv(\vx_t,t)$ from $t=1$ to $t=0$ , for example, with Euler steps, $\vx_{t_i} = \vx_{t_{i-1}} + \vv(\vx_{t_{i-1}},t_{i-1})\,\Delta t$ where $\Delta t := t_i-t_{i-1}$. 
While these models capture high-quality multimodal action distributions, accurate integration typically requires many denoising steps, motivating few-step distillation for faster inference.

\textbf{Vision-Language-Action Models.}
Vision-Language-Action (VLA) models~\citep{o2024open,black2024pi_0,intelligence2025pi_, bjorck2025gr00t,wu2026pragmatic,liu2024rdt,team2024octo} couple a pretrained vision-language backbone with an action decoding head to map an observation %
and a language instruction to a distribution over robot actions, most commonly via a flow or diffusion action expert~\citep{chi2023diffusionpolicy,black2024pi_0} trained with flow matching loss. Recent foundation VLA models illustrate the scale required for strong multi-task generalization, typically exceeding 2B parameters and trained on over 10{,}000 hours of multimodal robot data spanning multiple embodiments~\citep{black2024pi_0,intelligence2025pi_,wu2026pragmatic}, but remain expensive to deploy. 
Several works~\citep{team2024octo,shukor2025smolvla,wen2024tinyvla,reuss2025flower} instead train smaller VLAs; SmolVLA~\citep{shukor2025smolvla}, for example, reduces model size to 450 million parameters by using a small VLM and layer skipping. 
However, it still requires $N{=}10$ denoising steps per action chunk and does not benefit from the capabilities of larger foundation policies.

\textbf{Flow Map Models.}
A wide range of few-step generative methods aim to reduce iterative sampling cost.
Distillation-based approaches~\citep{sabour2025align, boffi2024flow}, including Consistency Models (CM)~\citep{song2023consistency}, and Consistency Trajectory Models (CTM)~\citep{kim2024consistency}, distill a pre-trained multi-step teacher into a few-step student, while score-distillation methods directly match teacher and student distributions~\citep{yin2024one,yin2024improved}.
Distillation-free approaches, such as MeanFlow (MF)~\citep{geng2025mean} and Shortcut Models~\citep{frans2025one}, instead train from scratch using self-consistency constraints. %
A flow map model $\vf_\theta(\vx,s,t)$ parameterizes the state-to-state transition over an arbitrary interval $[s,t]\subset[0,1]$, so that a single jump $\vf_\theta(\vx_s,s,t)=\vx_s+(t-s)\vu_\theta(\vx_s,s,t)$ at the interpolant $\vx_s=(1-s)\vx_0+s\vx_1$ replaces an entire sequence of Euler substeps. For example, \cite{frans2025one} combines flow matching loss with the self-consistency constraint at the midpoint time $m = ({s+t})/{2}$: 
\begin{equation}
    \mathcal{L}(\theta)=
    \mathbb{E}_{s, t, \vx_1, \vx_0}
    \left[ 
    || \vv_\theta(\vx_s,s)-(\vx_1 - \vx_0)||_2^2+ ||\vf_\theta(\vx_s,s,t)- \vf_{\theta^-}(\vf_{\theta^-}(\vx_s,s,m),m,t)||_2^2
    \right] 
    \label{eq:self-consistency}
\end{equation}
The first term enforces standard flow matching, while the second encourages the direct transition from $s$ to $t$ to agree with two shorter transitions composed through $m$.
In robot manipulation, these ideas have been adapted to accelerate diffusion- and flow-based policies such as Diffusion Policy~\citep{chi2023diffusionpolicy}. Consistency Policy (CP)~\citep{prasad2024consistency} adapts CTM~\citep{kim2024consistency}; OneDP~\citep{wang2024one} and SDM Policy~\citep{jia2024score} use distribution matching for one-step action generation; and recent methods learn average velocity fields~\citep{geng2025mean,geng2026improved,sheng2026mp1,guo2026reactvla} or shortcut objectives~\citep{frans2025one,luan2026snapflow,yan2025maniflow,li2026one} regularized by flow matching. 
For instance, SnapFlow~\citep{luan2026snapflow} constructs its target from two Euler steps but trains only on fixed endpoints, while OFP~\citep{li2026one} derives its target from nested sub-intervals with a self-guided regularizer.
Despite these advances, few-step policies often suffer from training instability and degraded quality as the number of denoising steps decreases. This trade-off is particularly critical in robot manipulation, where low latency must be achieved without sacrificing the precision and stability required for physical execution, and still remains insufficiently addressed~\citep{shukor2025smolvla}.

\textbf{On-policy Distillation.}
Given a pretrained teacher diffusion model, on-policy distillation (OPD) supervises the student to match the teacher's one-step transition kernel $p_T(\cdot\mid \vx_{t_i})$ along the student's own rollout $\vx_{0:N}\sim p_{S,\theta}$, by minimizing the reverse KL objective~\citep{li2026diffusionopd}. Since the student and teacher share the same diffusion coefficient, resulting in identical covariance, its per-step KL reduces to a closed-form squared distance between the means, $\mathrm{KL}(p_S\|p_T) \propto \|\mu_S(\vx_{t_i};\theta)-\mu_T(\vx_{t_i})\|_2^2$. Despite providing dense supervision, this objective necessitates evaluating the teacher at every single denoising step, resulting in computational inefficiency. FastOPD circumvents this training bottleneck by supervising the student at a single on-policy state per trajectory. 
Formulated within a few-step flow map regime, the distilled policy is naturally equipped for real-time robot control.

\section{FastOPD: Fast On-Policy Distillation}
\label{sec:method}
\begin{figure}[t] 
\centering 
\includegraphics[width=\linewidth]{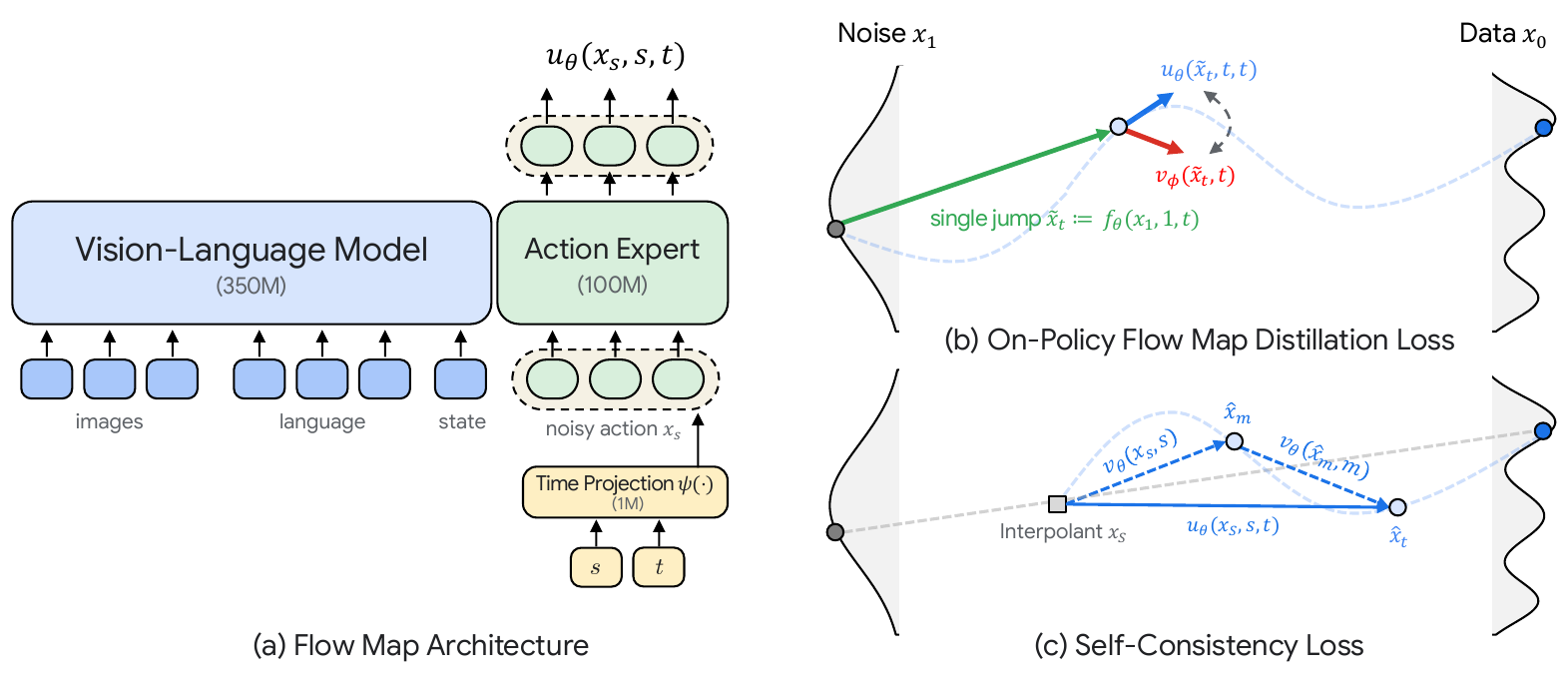} 
\vspace{-8mm}
\caption{
\textbf{Overview of FastOPD.}
(a) {Flow Map Architecture}: a frozen vision-language backbone with a lightweight action expert. A time projection layer $\psi(\cdot)$ fuses $(s,t)$ to condition the action expert on the two-time velocity $\vu_\theta(\vx_s,s,t)$.
(b) {On-Policy Flow Map Distillation Loss}: at a single on-policy jump $\tilde{\vx}_t=\vf_\theta(\vx_1,1,t)$, the student's diagonal velocity is matched to the teacher's. 
(c) {Self-Consistency Loss}: a direct transition $\vu_\theta(\vx_s,s,t)$ is trained to agree with two steps through the midpoint $m$.
}
\label{fig:overview}
\vspace{-5mm}
\end{figure}
\textbf{Overview.}
Our goal is to efficiently distill a large pretrained teacher $\vv_\phi$ into a compact flow map student $\vu_\theta$, while achieving performance on par with the teacher. Specifically, we parameterize the flow map as $\vf_\theta(\vx,s,t)=\vx+(t-s)\vu_\theta(\vx,s,t)$, which satisfies the boundary condition $\vf_\theta(\vx,s,s)=\vx$. We denote its diagonal by $\vv_\theta(\vx,t):=\vu_\theta(\vx,t,t)$, which corresponds to the instantaneous velocity. As illustrated in Fig.~\ref{fig:overview}, we embed $s$ and $t$ independently via sinusoidal positional encoding and fuse them through an additional projection layer $\psi(\cdot)$. Here, we omit the condition $c$ for brevity.

The key idea of FastOPD is to match the teacher only at sampled on-policy states and propagate this local velocity supervision to finite-step flow maps through self-consistency.
Accordingly, FastOPD jointly optimizes on-policy velocity matching and finite-interval self-consistency:
\begin{align}
    \mathcal{L}(\theta) = \underbrace{\mathbb{E}\Big[ \big\|\vu_\theta(\vx_s, s, t) - \tfrac{1}{2}\big(\mathrm{sg}[\vv_{\theta}(\vx_s,s)] + \mathrm{sg}[\vv_{\theta}(\hat{\vx}_m, m)]\big) \big\|^2 \Big]}_{\text{finite-interval self-consistency}} + \lambda \underbrace{\mathbb{E}\Big[\big\| \vv_\theta(\tilde{\vx}_t, t) - \vv_\phi(\tilde{\vx}_t, t) \big\|^2 \Big]}_{\text{on-policy flow map distillation}}.
\label{eq:fastopd}
\end{align}
Here, $\mathrm{sg}[\cdot]$ denotes the stop-gradient operator, and $\tilde{\vx}_t$ denotes a sampled on-policy state, while $\hat{\vx}_m$ denotes the midpoint state induced by the student. We next describe how the two terms implement local teacher matching and its propagation to finite-step flow maps.

\textbf{On-policy Flow Map Distillation.}
We first consider how to efficiently transfer the teacher dynamics to the student. Conventional flow-based OPD evaluates the teacher at every denoising step along the student trajectory, resulting in substantial computational overhead. %
Instead, we supervise the student at a single sampled on-policy state per trajectory{ (see Fig.~\ref{fig:overview}(b))}. 
Specifically, we obtain the state with a single flow-map transition from noise, $\tilde{\vx}_t = \mathrm{sg}\big[\vf_\theta(\vx_1, 1, t)\big]$, which induces the on-policy marginal $\tilde{\vx}_t\sim p_{\theta}$. We then match the teacher and student transitions only at this sampled timestep $t$.
Let $p_{\theta}(\cdot\mid\tilde{\vx}_t)$ and $p_\phi(\cdot\mid\tilde{\vx}_t)$ denote the corresponding one-step transition kernels of the student and teacher, respectively. We formulate the \textit{single-state transition matching} objective as
\begin{equation}
\mathbb{E}_{t\sim p(t)} \mathbb{E}_{\tilde{\vx}_t \sim p_{\theta}} \left[ \text{KL} \left( p_\theta(\cdot\mid\tilde{\vx}_t) \| p_\phi(\cdot\mid\tilde{\vx}_t) \right) \right].
\label{eq:opfd-kl}
\end{equation}
For a flow map, the instantaneous velocity on the diagonal coincides with the underlying velocity field under the tangent condition~\citep{boffi2025build} \footnote{$\lim_{s\to t}\partial_t \vf_\theta(\vx,s,t) = \vu_\theta(\vx,t,t) = \vv_\theta(\vx,t)$}. Using this identity, the reverse KL in \eqref{eq:opfd-kl} can be expressed directly in terms of the student and teacher velocity fields, as shown in the following proposition.
\begin{restatable}[On-policy Flow Map Distillation]{prop}{opfdprop}
\label{prop:prop1}
Let $\sigma_t>0$ be a diffusion coefficient.
For any $t\in(0,1)$ and step $\Delta t<0$, the
Euler--Maruyama kernels induced by
$\vu_\theta(\cdot,t,t)$ and $\vv_\phi(\cdot,t)$ satisfy
\begin{equation}
\KL\left(
p_\theta(\cdot\mid\tilde{\vx}_t)
\,\middle\|\,
p_\phi(\cdot\mid\tilde{\vx}_t)
\right)
=
w_\sigma(t,\Delta t)
\left\|
\vu_\theta(\tilde{\vx}_t,t,t)
-
\vv_\phi(\tilde{\vx}_t,t)
\right\|_2^2,
\end{equation}
where $w_\sigma(t,\Delta t)
=
\frac{|\Delta t|}{2\sigma_t^2}
\left(
1+\frac{\sigma_t^2(1-t)}{2t}
\right)^2$.
\end{restatable}
Proposition~\ref{prop:prop1} reduces the single-state transition matching objective to a weighted velocity regression. We therefore define the practical on-policy flow map distillation (OPFD) objective as
\begin{equation}
    \mathcal{L}_{\text{OPFD}}(\theta)=\mathbb{E}_{c,t \sim p(t)}\left[ \mathbb{E}_{\tilde{\vx}_t \sim p_{\theta}} \left[ \|\vv_\theta(\tilde{\vx}_t,t)-\vv_\phi(\tilde{\vx}_t,t) \|^2_2\right]\right].
 \label{eq:opfd}
\end{equation}
\Eqref{eq:opfd} can be viewed as the pointwise version of the original distribution matching in \eqref{eq:opfd-kl}, since it is characterized by a deterministic flow ODE. We set the weighting function $w_\sigma(t,\Delta t)$ to 1 in practice, following \cite{li2026diffusionopd}. Full derivation is provided in Appendix~\ref{app:proof}.

\textbf{Self-consistency Property.} 
OPFD transfers the teacher velocity to the student's 
diagonal flow map  $\vu_\theta(\vx,t,t)$, but our goal is to learn a flow map over arbitrary finite interval $(s,t)$ at the same time.
We therefore propagate this local teacher supervision to finite-step transitions through self-consistency (see Fig.~\ref{fig:overview}(c)).
Following Shortcut Models~\citep{frans2025one}, a direct transition from $s$ to $t$ should agree with two consecutive shorter transitions through the midpoint $m=(s+t)/2$.
Unlike the original shortcut formulation, we construct the target using the \emph{instantaneous} velocities at the diagonals of the two sub-intervals, $\vv_\theta(\vx_s,s)$ and $v_\theta(\hat{\vx}_m, m)$, as follows:
\begin{equation}
    \mathcal{L}_{\text{SC}}(\theta)=\mathbb{E}_{c, s,t,\vx} \left[ \left\|\vu_\theta(\vx_s,s,t)-\frac{1}{2}\left[\mathrm{sg}\big[\vv_{\theta}(\vx_s,s)\big]+\mathrm{sg}\big[\vv_{\theta}(\hat{\vx}_{m},m)\big]\right]\right\|^2 \right]
\end{equation}
where $\hat{\vx}_m := \vx_s + (m-s)\vv_\theta(\vx_s,s)$.  

The resulting target is anchored to teacher-supervised quantities at every interval. 
Note that the diagonal terms coincide with the quantities directly supervised by OPFD. 
Accordingly, we define the training objective as
\begin{equation}
    \mathcal{L}_{\text{FastOPD}} = \mathcal{L}_{\text{SC}} + \lambda\mathcal{L}_{\text{OPFD}}.
\end{equation}
In the following theorem, we demonstrate that the student implicitly recovers the teacher-induced flow map.
Let $\vv_\phi(\vx,t)$ be the teacher velocity field, and we define the teacher-induced ideal shortcut as
\begin{equation}
{\vu_\phi}(\vx,s,t) = \tfrac{1}{2}\big[\vv_\phi(\vx,s) + \vv_\phi({\bar{\vx}}_m, m)\big], \quad
\bar{\vx}_m = \vx + (m-s)\,\vv_\phi(\vx,s),
\end{equation}
with ${\vu}_\phi(\vx,t,t)=\vv_\phi(\vx,t)$. We then define the corresponding ideal flow map as $\vf_\phi(\vx,s,t)=\vx+(t-s)\vu_\phi(\vx,s,t)$.

\begin{restatable}[]{prop}{boundprop}
\label{prop:prop2}
    Let $q_\theta$ and $\rho_\theta$ be the sampling distributions of $\mathcal{L}_{\text{SC}}$ and $\mathcal{L}_{\text{OPFD}}$. Assume that $\vv_\phi(\cdot,t)$ is $L$-Lipschitz and that $(s,\vx_s)$ and $(m,\hat{\vx}_m)$ under $q_\theta$
have density ratio at most $\kappa$ with respect to $\rho_\theta$. Then,
    \begin{equation}
        D_{q_\theta}(\vu_\theta, \vu_\phi) \leq C_\lambda\mathcal{L}_{\text{FastOPD}},
    \end{equation}
    where $D_{q_\theta}(\vu_\theta, \vu_\phi):=\mathbb{E}_{(s,t,\vx)\sim q_\theta} \|\vu_\theta(\vx, s, t) - \vu_\phi(\vx, s, t) \|_2^2$ and $C_\lambda = \max\Big\{4,\ \frac{\kappa(2+L^2/4)}{\lambda}\Big\}$.
\end{restatable}

Thus, OPFD transfers the teacher's instantaneous velocity to the student, while self-consistency propagates this agreement from the diagonal velocity to finite-step shortcut maps.

\begin{restatable}[]{theorem}{main}
\label{theorem:main}
Let $p_\theta$ and $p_{\phi}$ denote the terminal distributions induced by sampling with $\vf_\theta$ and $\vf_{\phi}$, respectively, from the same initial noise on a fixed sampling schedule $1=t_0>t_1>\cdots>t_K=0$, and let $\vx^\theta_{t_k}$ denote the intermediate states of the student. Suppose that the assumptions of Proposition~\ref{prop:prop2} hold and there exists finite $\kappa_{\mathrm{roll}}>0$ such that, for every $k$,
\begin{equation}
D_k:=\mathbb{E}\big\|\vu_\theta(\vx^\theta_{t_k},t_k,t_{k+1})-\vu_\phi(\vx^\theta_{t_k},t_k,t_{k+1})\big\|_2^2
    \le \kappa_{\mathrm{roll}}\, D_{q_\theta}(\vu_\theta,\vu_\phi).
\end{equation}
Then,
\begin{equation}
    W_2(p_\theta, p_{\phi}) \leq C \sqrt{\mathcal{L}_{\text{FastOPD}}},
\end{equation}
where $C>0$ depends only on $L$, $\kappa$, $\kappa_{\mathrm{roll}}$, $\lambda$, and the sampling schedule.
\end{restatable}
Consequently, $W_2(p_\theta, p_{\phi}) \to 0$ as $\mathcal{L}_{\text{FastOPD}} \to 0$.
Since $\vu_\phi$ is built entirely from the pretrained teacher's velocity field, $p_{\phi}$ is the few-step distribution induced by the teacher-induced flow map on the same sampling schedule. The theorem therefore shows that FastOPD allows the compact student to implicitly recover this distribution without explicitly constructing a shortcut teacher.
The detailed algorithm for FastOPD is presented in Algorithm~\ref{alg:fastopd}.

\begin{algorithm}[t]
\caption{Training Algorithm for FastOPD}
\label{alg:fastopd}
\begin{algorithmic}[1]
\Require Dataset $\mathcal{D}$; Pretrained teacher $\vv_\phi$; student flow map $\vu_\theta$ with time embedding layer $\psi(\cdot)$; OPFD weight $\lambda$; Sampling time distribution $p$

\For{each training iteration}
    \State $\vx_1\sim\mathcal N(\mathbf 0,\mathbf I),\ \tau\sim p,\ (\vx_0,c)\sim \mathcal{D}$
    \State $\tilde{\vx}_{\tau} \leftarrow \mathrm{sg}\big[\vf_{\theta}(\vx_1,1,\tau)\big]$
    \State $\mathcal{L}_{\text{OPFD}} \leftarrow \|\vv_\theta(\tilde{\vx}_{\tau},\tau \mid c) - \vv_\phi(\tilde{\vx}_{\tau},\tau \mid c)\|_2^2$ \Comment{single on-policy state supervision}
    \State $s,t\sim p$ with $s>t$
    \State $m \leftarrow (s+t)/2$, $\vx_s \leftarrow (1-s)\vx_0+s\vx_1$
    \State $\hat{\vx}_m \leftarrow \vx_s + (m-s)\,\mathrm{sg}\big[\vv_\theta(\vx_s,s\mid c)\big]$
    \State $\mathcal{L}_{\text{SC}} \leftarrow \big\|\vu_\theta(\vx_s,s,t \mid c) - \tfrac{1}{2}\big[\mathrm{sg}[\vv_\theta(\vx_s,s\mid c)] + \mathrm{sg}[\vv_\theta(\hat{\vx}_m,m\mid c)]\big]\big\|_2^2$
    \State $\mathcal{L}_{\text{FastOPD}} \leftarrow \mathcal{L}_{\text{SC}} + \lambda\,\mathcal{L}_{\text{OPFD}}$
    \State $\theta \leftarrow \theta - \eta\,\nabla_\theta \mathcal{L}_{\text{FastOPD}}$
\EndFor
\end{algorithmic}

\end{algorithm}

\section{Experiments}
Our experiments span two multitask simulation benchmarks (Sec.~\ref{sec:simulation}) and one real-robot task (Sec.~\ref{sec:real}).
We freeze the vision--language backbone and fine-tune only the action expert and projections with the additional time projection layer described in Sec.~\ref{sec:method}.
We evaluate each method using 1, 2, 4, and 10 denoising steps and compare it against the teacher and the initial student with 50 trials per task.
Full experimental details are provided in Appendix~\ref{app:exp-details}.

\subsection{Simulation Experiments}
\label{sec:simulation}
\textbf{Benchmarks.} 
We evaluate FastOPD on LIBERO~\citep{liu2023libero} and RoboTwin~2.0~\citep{chen2025robotwin}.
LIBERO evaluates single-arm manipulation across four standard evaluation suites---Spatial, Object, Goal, and Long---comprising 40 tasks in total.
RoboTwin~2.0~\citep{chen2025robotwin} comprises 50 challenging bimanual manipulation tasks, with 2,500 demonstrations collected in clean scenes and 25,000 demonstrations collected under heavy scene randomization.
We report results in both clean and randomized settings on unseen environments.

\textbf{Baselines.}
For simulation experiments, we compare our method with three representative approaches to few-step generation: Consistency Trajectory Model (CTM)~\citep{kim2024consistency},
Distribution Matching Distillation (DMD)~\citep{yin2024one,yin2024improved},
and improved Mean Flow (iMF)~\citep{geng2026improved}.
Since iMF was originally formulated for training from scratch, we adapt it to
the distillation setting by replacing its self-referential instantaneous
velocity target with the teacher velocity.
Except for DMD, a one-step generator, all baselines are designed as the flow-map architecture.

\textbf{Teacher and Student Models.}
We use SmolVLA~\citep{shukor2025smolvla} as the student model on both benchmarks,
paired with $\pi_{0.5}$~\citep{intelligence2025pi_} on LIBERO, LingBot-VLA~\citep{wu2026pragmatic} and Fast-WAM~\citep{yuan2026fast} on RoboTwin~2.0.
All models are initialized from public benchmark-specific checkpoints.
\vspace{10pt}

\begin{table}[t]
\centering
\begin{minipage}[t]{0.485\linewidth}
\vspace{0pt}
\caption{
Success rates (\%) on LIBERO across spatial, object, goal,
and long-horizon tasks. $\dagger$: Results reported in the original papers.
}
\label{tab:main_results_libero}
\centering
\scriptsize
\setlength{\tabcolsep}{1.5pt}
\renewcommand{\arraystretch}{1.18} %
\begin{tabular*}{\linewidth}{@{\extracolsep{\fill}}lccccccc@{}}
\toprule
\textbf{Method} & \textbf{Param.} & \textbf{NFE}
& \textbf{Spat.} & \textbf{Obj.} & \textbf{Goal} & \textbf{Long} & \textbf{Avg.} \\
\midrule
\multicolumn{8}{l}{\textbf{Diffusion/Flow-based Policies}} \\
\midrule
DP$^\dagger$        & --   & 100 & 78.3 & 92.5 & 68.3 & 50.5 & 72.4 \\
Octo-Base$^\dagger$ & 93M  & 10  & 78.9 & 85.7 & 84.6 & 51.1 & 75.1 \\
$\pi_{0}$$^\dagger$ & 3.3B & 10  & 97.4 & 98.4 & 97.6 & 93.0 & 96.6 \\
\midrule
\multicolumn{8}{l}{\textbf{Teacher \& Base Student Models}} \\
\midrule
$\pi_{0.5}$ & 3.3B & 10 & 98.2 & 99.4 & 97.4 & 94.8 & 97.5 \\
\multirow[t]{4}{*}{SmolVLA} & \multirow[t]{4}{*}{450M}
& 1  & 74.6 & 73.2 & 77.4 & 51.2 & 69.1 \\
& & 2  & 76.8 & 74.4 & 79.8 & 54.4 & 71.4 \\
& & 4  & 79.0 & 74.8 & 79.2 & 58.0 & \underline{72.8} \\
& & 10 & 77.2 & 75.0 & 79.4 & 54.6 & \underline{71.6} \\
\midrule
\multicolumn{8}{l}{\textbf{Few-step Distillation}} \\
\midrule
DMD & 450M & 1 & 82.4 & 87.2 & 75.6 & 47.2 & \underline{73.1} \\
\multirow[t]{4}{*}{CTM} & \multirow[t]{4}{*}{451M}
& 1  & 78.0 & 86.6 & 70.8 & 41.4 & 69.2 \\
& & 2  & 75.4 & 84.2 & 70.4 & 43.4 & 68.4 \\
& & 4  & 68.4 & 77.0 & 71.2 & 45.0 & 65.4 \\
& & 10 & 72.0 & 81.6 & 72.4 & 44.4 & 67.6 \\
\multirow[t]{4}{*}{iMF} & \multirow[t]{4}{*}{451M}
& 1  & 70.8 & 78.6 & 75.6 & 30.6 & 63.9 \\
& & 2  & 76.0 & 90.6 & 75.6 & 46.4 & \underline{72.2} \\
& & 4  & 68.2 & 84.8 & 67.0 & 38.0 & 64.5 \\
& & 10 & 65.2 & 83.6 & 66.6 & 35.2 & 62.7 \\
\multirow[t]{4}{*}{\textbf{Ours}} & \multirow[t]{4}{*}{451M}
& 1  & 84.2 & 94.2 & 82.6 & 48.2 & \textbf{77.3} \\
& & 2  & 83.6 & 96.8 & 87.2 & 59.6 & \textbf{81.8} \\
& & 4  & 83.0 & 95.0 & 85.6 & 57.4 & \textbf{80.3} \\
& & 10 & 80.4 & 94.0 & 84.4 & 57.2 & \textbf{79.0} \\
\bottomrule
\end{tabular*}
\end{minipage}
\hfill
\begin{minipage}[t]{0.495\linewidth}
\vspace{0pt}
\caption{
Success rates (\%) on RoboTwin~2.0 under unseen clean
and randomized settings.  $\dagger$: Results reported in the original papers.
}
\label{tab:main_results_robotwin2}
\centering
\scriptsize
\setlength{\tabcolsep}{1.5pt}
\renewcommand{\arraystretch}{0.94}
\begin{tabular*}{\linewidth}{@{\extracolsep{\fill}}lccccc@{}}
\toprule
\textbf{Method} & \textbf{Param.} & \textbf{NFE}
& \textbf{Clean} & \textbf{Rand.} & \textbf{Avg.} \\
\midrule
\multicolumn{6}{l}{\textbf{Diffusion/Flow-based Policies}} \\
\midrule
DP$^\dagger$          & 0.46B & 10 & 28.0 & 0.6  & 14.3 \\
RDT$^\dagger$         & 1.2B  & 10 & 34.5 & 13.7 & 24.1 \\
$\pi_0$$^\dagger$     & 3.3B  & 10 & 46.4 & 16.3 & 31.4 \\
$\pi_{0.5}$ & 3.3B  & 10 & 70.7 & 46.0 & 58.4 \\
\midrule
\multicolumn{6}{l}{\textbf{Teacher \& Base Student Models}} \\
\midrule
LingBot-VLA & 4.2B & 10 & 86.8 & 84.9 & 85.8 \\
Fast-WAM    & 6B   & 10 & 91.9 & 91.8 & 91.9 \\
\multirow[t]{4}{*}{SmolVLA} & \multirow[t]{4}{*}{450M}
& 1  & 35.2 & 35.4 & 35.3 \\
& & 2  & 55.0 & 56.1 & 55.6 \\
& & 4  & 59.8 & 61.0 & \textbf{60.4} \\
& & 10 & 59.4 & 61.3 & \textbf{60.4} \\
\midrule
\multicolumn{6}{l}{\textbf{Few-step Distillation}} \\
\midrule
DMD & 450M & 1 & 32.5 & 32.1 & 32.3 \\
\multirow[t]{4}{*}{CTM} & \multirow[t]{4}{*}{451M}
& 1  & 15.5 & 15.6 & 15.6 \\
& & 2  & 4.8  & 15.8 & 10.3 \\
& & 4  & 18.7 & 16.0 & 17.3 \\
& & 10 & 36.1 & 34.5 & 35.3 \\
\multirow[t]{4}{*}{iMF} & \multirow[t]{4}{*}{451M}
& 1  & 43.2 & 42.4 & 42.8 \\
& & 2  & 51.6 & 51.1 & 51.4 \\
& & 4  & 50.8 & 51.0 & 50.9 \\
& & 10 & 50.6 & 51.4 & 51.0 \\
\multirow[t]{4}{*}{\textbf{Ours (LingBot-VLA)}} & \multirow[t]{4}{*}{451M}
& 1  & 51.5 & 50.9 & \textbf{51.2} \\
& & 2  & 57.1 & 58.3 & \textbf{57.7} \\
& & 4  & 58.6 & 61.2 & \underline{59.9} \\
& & 10 & 56.9 & 59.6 & 58.2 \\
\multirow[t]{4}{*}{\textbf{Ours (Fast-WAM)}} & \multirow[t]{4}{*}{451M}
& 1  & 51.1 & 49.8 & \underline{50.5} \\
& & 2  & 55.5 & 57.2 & \underline{56.3} \\
& & 4  & 58.5 & 59.2 & 58.8 \\
& & 10 & 58.3 & 59.6 & \underline{58.9} \\
\bottomrule
\end{tabular*}

\end{minipage}

\vspace{0pt}
\caption{End-to-end inference latency (ms) on LIBERO and RoboTwin~2.0.}
\label{tab:latency}
\centering
\scriptsize
\setlength{\tabcolsep}{3pt}
\renewcommand{\arraystretch}{1.0}
\begin{tabular*}{\linewidth}{@{\extracolsep{\fill}}lcccc@{\hspace{12pt}}lcccc@{}}
\toprule
\multicolumn{5}{c}{\textbf{LIBERO}} & \multicolumn{5}{c}{\textbf{RoboTwin~2.0}} \\
\cmidrule(lr){1-5}\cmidrule(l){6-10}
\textbf{Model} & \textbf{NFE=1} & \textbf{2} & \textbf{4} & \textbf{10}
& \textbf{Model} & \textbf{NFE=1} & \textbf{2} & \textbf{4} & \textbf{10} \\
\midrule
$\pi_{0.5}$ (teacher) & 155 & 171 & 203 & 301 & LingBot-VLA (teacher) & 151 & 197 & 291 & 571 \\
                      &    &    &    &     & Fast-WAM (teacher)    & 147 & 186 & 264 & 503 \\
SmolVLA               & 49 & 65 & 97 & 193 & SmolVLA              & 56 & 72 & 103 & 199 \\
\textbf{Ours}         & 50 & 66 & 98 & 193 & \textbf{Ours}        & 57 & 74 & 106 & 202 \\
\bottomrule
\end{tabular*}
\vspace{-15pt}
\end{table}

\begin{wrapfigure}{r}{0.55\linewidth}
    \centering
    \vspace{-15pt}  
    \includegraphics[width=\linewidth]{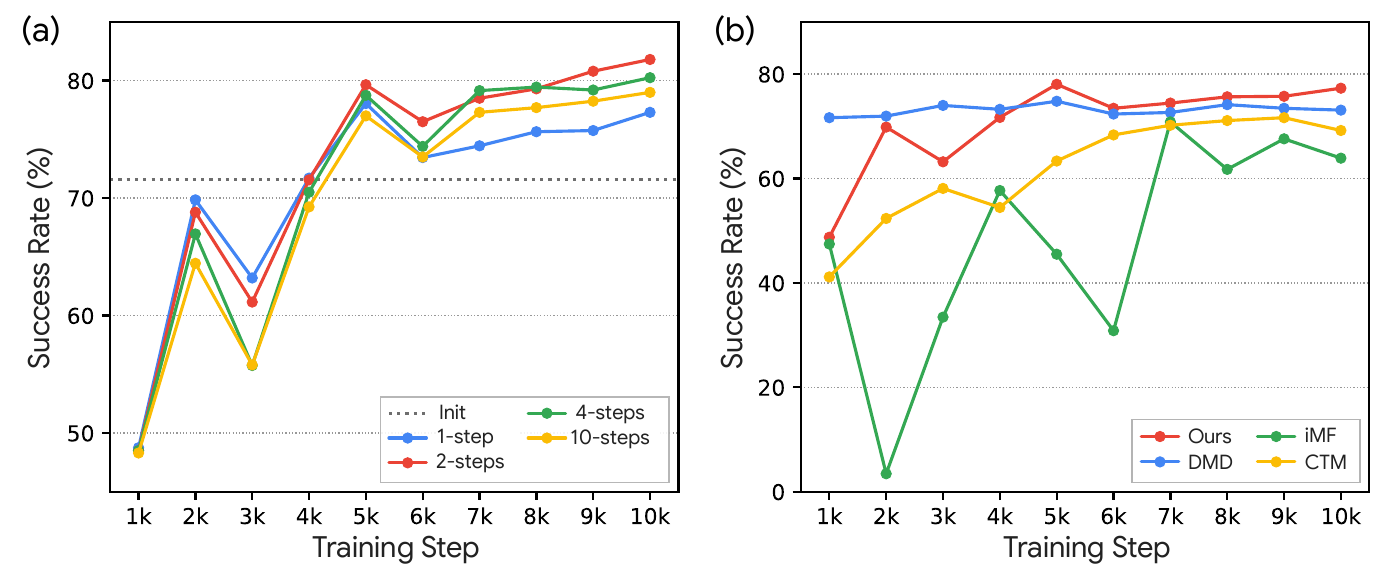}
    \vspace{-7mm} 
    \caption{Success rate (\%) on LIBERO during training: (a) FastOPD at different sampling steps, (b) comparison with baselines at 1 step.}
    \label{fig:qual-sr-steps}
    \vspace{-4mm}
\end{wrapfigure}
\textbf{LIBERO.} In Tab.~\ref{tab:main_results_libero} and Tab.~\ref{tab:latency}, we report the average success rates and inference latency on the LIBERO benchmark. 
Our method outperforms all few-step baselines and the initial student model at every denoising step in average success rate.
As shown in Fig.~\ref{fig:qual-sr-steps}, we observe a consistent performance improvement
across all denoising steps, including at 10 steps. 
FastOPD further achieves the best success rates of 81.8\% with only 2 sampling steps.
In particular, our method shows better performance than the 10-step SmolVLA baseline with only a single step, while reducing inference latency from 193\,ms to 50\,ms.
Together, the results demonstrate that FastOPD substantially improves the accuracy--latency trade-off of the compact student.
We examine the contributions of each loss component to these accuracy gains in Sec.~\ref{sec:ablation}.

\textbf{RoboTwin~2.0.}
Tab.~\ref{tab:main_results_robotwin2} and Tab.~\ref{tab:latency} summarize the results. Notably, our method consistently achieves performance gains in the few-step regime, improving the single-step success rate of the base student from 35.3\% to 51.2\%, and outperforms all few-step distillation baselines at every sampling step.
It does so at a latency comparable to SmolVLA, while being $10.0\times$ faster than the LingBot-VLA teacher.
At 4 and 10 steps, FastOPD performs slightly below the base student, indicating that its benefit lies mainly in
the few-step regime where the base student degrades.

\textbf{Distilling a World Action Model.}
Recent World Action Models (WAMs) build on pretrained video generation models and have emerged as an alternative to foundation VLAs~\citep{kim2026cosmos,ye2026world,yuan2026fast}. To examine whether FastOPD extends beyond VLA teachers, we distill Fast-WAM~\citep{yuan2026fast}, a 6B-parameter video-based teacher, into the same 451M SmolVLA student with an identical training recipe. 
As shown in Tab.~\ref{tab:main_results_robotwin2}, the student distilled from Fast-WAM achieves a 50.5\% average success rate at a single denoising step, 15.2 percentage points over the initial SmolVLA baseline (35.3\%). Furthermore, it reaches 56.3\% at two steps, successfully transferring knowledge despite the substantial differences in architecture and pretraining between the teacher and the student.
This demonstrates that FastOPD is a versatile approach that applies not only to VLAs but also to WAMs, since it only requires the teacher's velocity field on the student's own action trajectories, regardless of how the teacher is built.

\begin{wrapfigure}{r}{0.55\linewidth}
    \centering
    \vspace{-15pt}
    \includegraphics[width=\linewidth]{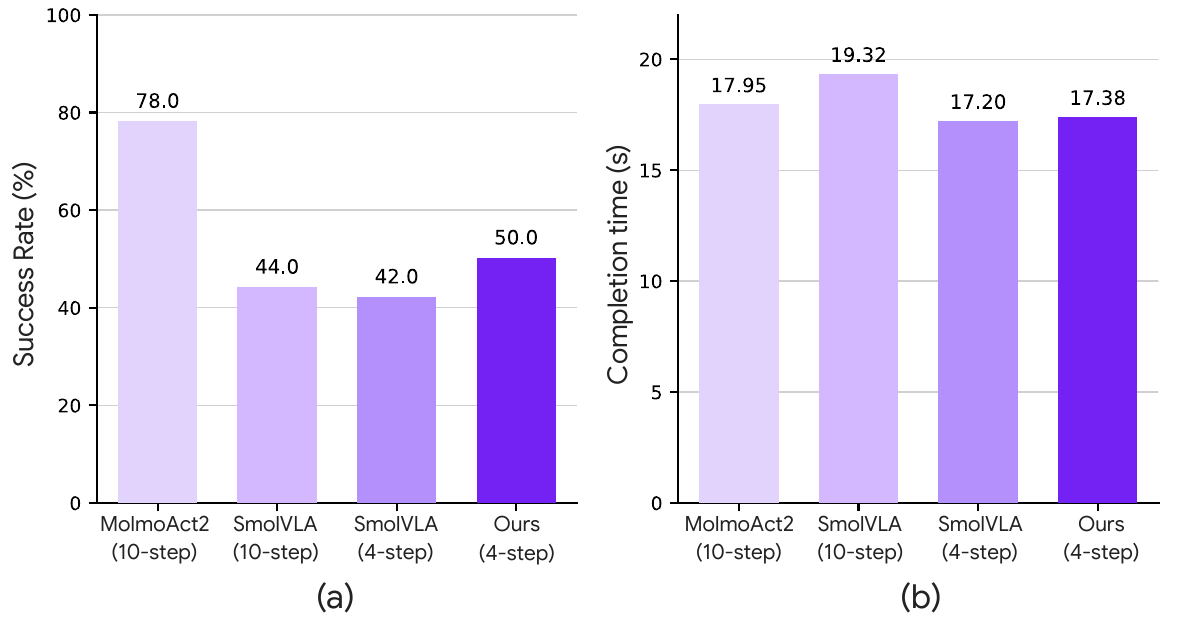}
    \vspace{-20pt}
    \caption{Real-world results on the \texttt{pnp-plate} task. (a) Success rate (\%). (b) Task completion time (s), averaged over successful episodes.}
    \label{fig:real-exp}
    \vspace{1pt}
    \includegraphics[width=\linewidth]{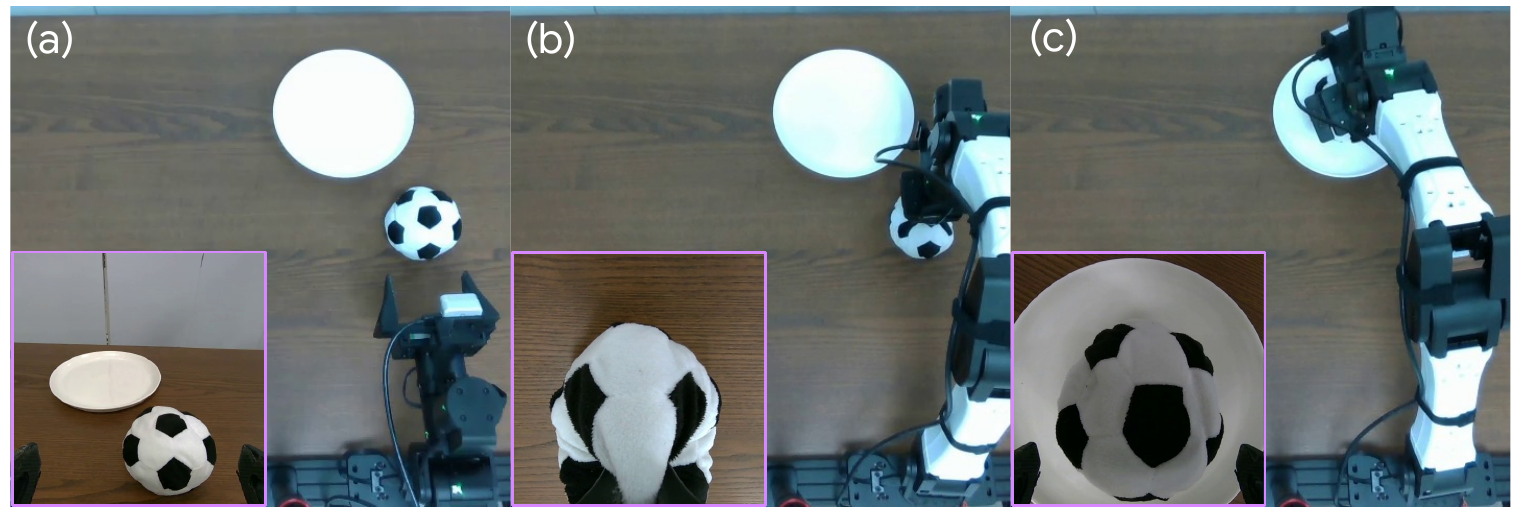}
    \caption{FastOPD on \texttt{pnp-plate}: (a) start, (b) transport, and
(c) placement. Each panel shows the top view with the wrist view inset (lower left).}
    \label{fig:real-rollout}
    \vspace{-5mm}
\end{wrapfigure}
\subsection{Real-World Experiments}
\label{sec:real}
To validate its real-world applicability, we further design a real-world robot scenario, \texttt{pnp-plate}, in which the robot picks up the ball and places it on the plate.
Experiments use the right arm of a bimanual YAM robot with a parallel-jaw
gripper and top-view and wrist-mounted cameras that provide RGB observations.
We train MolmoAct2~\citep{fang2026molmoact2actionreasoningmodels}, an action reasoning model that is suitable for real-world deployment, on our collected demonstrations for 14{,}000 iterations and use the resulting policy as the teacher.
We initialize the student from SmolVLA, warm it up on the collected demonstrations with
the standard flow matching objective (i.e., behavior cloning) for 10{,}000 iterations,
and subsequently distill the teacher into the student using FastOPD. Full experimental details
are provided in Appendix~\ref{app:real-setting}.

Fig.~\ref{fig:real-exp} and Fig.~\ref{fig:real-rollout} summarize the results.
With four denoising steps, FastOPD achieves a 50\% success rate, improving by 8 percentage points over
the 42\% of the base SmolVLA at the same number of steps. 
Notably, FastOPD shows faster and more accurate performance than the 10-step base student, improving the success rate from 44\% to 50\% while cutting the execution time from 19.32\,s to 17.38\,s.
These results demonstrate that FastOPD can transfer the capability of a larger teacher to a
compact student deployed on a physical robot with few steps.

\subsection{Ablation Study}
\label{sec:ablation}

\textbf{Training efficiency.} 
FastOPD reduces training overhead by supervising the student at a single sampled on-policy state rather than across all denoising steps.
As illustrated in Fig.~\ref{fig:figure1}, standard OPD~\citep{li2026diffusionopd} takes 9.51 hours per 1,000 training iterations while FastOPD requires only 1.43 hours, achieving a $5.7\times$ speedup.
This translates directly into accelerated convergence (Fig.~\ref{fig:figure1}). FastOPD achieves a 78.2\% 1-step success rate in just 10 hours, whereas conventional OPD requires roughly 57 hours to reach a similar level (79.1\%). Consequently, FastOPD matches the performance of conventional OPD at a fraction of the computational cost.

\textbf{Student Initialization.}
\label{sec:ode-init-ablation}
We study whether FastOPD can also start from a student that is initialized
in a fully on-policy manner. To do so, we initialize the student using 5,000 iterations of conventional per-step OPD~\citep{li2026diffusionopd} instead of flow-matching behavior cloning.
As shown in Tab.~\ref{tab:two_stage_ablation}, the OPD-initialized student degrades severely at 4 steps (44.2\%), likely due to the fixed 10-step discretization in OPD, whereas FastOPD recovers this case and improves all denoising steps.
 This indicates that OPD is a viable alternative to
flow-matching-based behavior cloning for the initialization, although it incurs a substantially higher
training cost, as discussed above (Fig.~\ref{fig:figure1}). It also shows that FastOPD is not tied to a specific initialization
and can be applied on top of students distilled by other methods.

\begin{table}[t]
\centering
\scriptsize
\setlength{\tabcolsep}{3pt}
\captionsetup{skip=3pt}
\sbox{\ablA}{%
\begin{tabular}{lcccc}
\toprule
\textbf{Method} & \textbf{1} & \textbf{2} & \textbf{4} & \textbf{10} \\
\midrule
FM init.    & 69.1 & 71.4 & 72.8 & 71.6 \\
\rowcolor{lavender}
\,+ FastOPD & 77.3 & 81.8 & 80.3 & 79.0 \\
OPD init.   & 72.7 & 78.2 & 44.2 & 78.6 \\
\rowcolor{lavender}
\,+ FastOPD & \textbf{80.5} & \textbf{83.0} & \textbf{82.2} & \textbf{81.7} \\
\bottomrule
\end{tabular}}%
\sbox{\ablB}{%
\begin{tabular}{lcccc}
\toprule
\textbf{Method} & \textbf{1} & \textbf{2} & \textbf{4} & \textbf{10} \\
\midrule
Base      & 69.1 & 71.4 & 72.8 & 71.6 \\
SC only   & 14.5 & 15.7 & 17.8 & 16.8 \\
OPFD only & 25.0 & 52.6 & 64.0 & 74.7 \\
\rowcolor{lavender} \textbf{FastOPD} & \textbf{77.3} & \textbf{81.8} & \textbf{80.3} & \textbf{79.0} \\
\bottomrule
\end{tabular}}%
\sbox{\ablC}{%
\begin{tabular}{ccccc}
\toprule
$\lambda$ & \textbf{1} & \textbf{2} & \textbf{4} & \textbf{10} \\
\midrule
1    & 77.9 & 82.0 & 83.6 & 82.1 \\
0.5  & 79.3 & 82.9 & \textbf{84.2} & \textbf{84.1} \\
\rowcolor{lavender} 0.1  & \textbf{80.5} & \textbf{83.0} & 82.2 & 81.7 \\
0.01 & 53.0 & 52.6 & 52.9 & 52.0 \\
\bottomrule
\end{tabular}}%
\def\ablgap{0.02\linewidth}%
\edef\ablscale{\fpeval{(\linewidth-2*\ablgap)/(\wd\ablA+\wd\ablB+\wd\ablC)}}%
\begin{minipage}[t]{\ablscale\wd\ablA}\vspace{0pt}\centering
\captionof{table}{Student initialization.}\label{tab:two_stage_ablation}
\resizebox{\linewidth}{!}{\usebox{\ablA}}
\end{minipage}\hfill
\begin{minipage}[t]{\ablscale\wd\ablB}\vspace{0pt}\centering
\captionof{table}{Loss components.}\label{tab:loss_ablation}
\resizebox{\linewidth}{!}{\usebox{\ablB}}
\end{minipage}\hfill
\begin{minipage}[t]{\ablscale\wd\ablC}\vspace{0pt}\centering
\captionof{table}{Effect of $\lambda$.}\label{tab:lambda_ablation}
\resizebox{\linewidth}{!}{\usebox{\ablC}}
\end{minipage}
\vspace{-4mm}
\end{table}

\textbf{Loss components.}
Tab.~\ref{tab:loss_ablation} shows that the two objectives are complementary.
Without teacher anchoring, SC alone collapses below 18\% and OPFD alone fails to capture off-diagonal flow dynamics, degrading few-step generation.  
Their combination reaches 81.8\% with two steps.
While \cite{luan2026snapflow} also anchors a fixed two-step Euler target with flow matching, FastOPD anchors to the teacher via OPFD and enforces consistency over arbitrary intervals, yielding an any-step flow map that inherits the teacher's dynamics (see Appendix~\ref{app:add-exp}).

\begin{wrapfigure}{r}{0.4\linewidth}
    \centering
    \vspace{-10pt}
    \includegraphics[width=\linewidth]{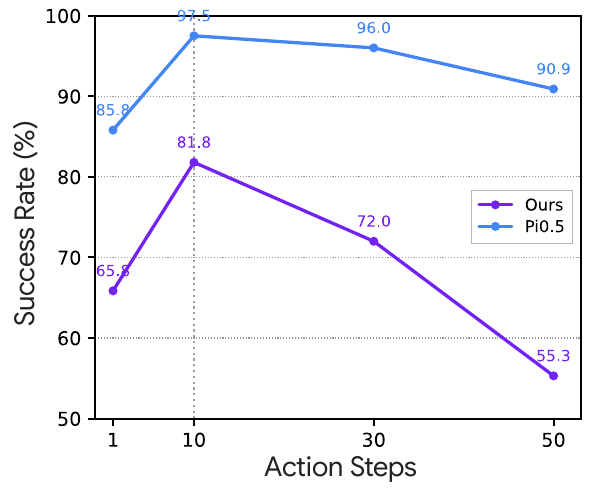}
    \vspace{-20pt}
    \caption{Ablation on action execution steps on LIBERO (NFE$=$2 for Ours and NFE$=$10 for $\pi_{0.5}$).}
    \label{fig:action-steps}
    \vspace{-15pt}
\end{wrapfigure}
\textbf{Hyperparameters.}
The proposed method introduces a single
hyperparameter $\lambda$, which controls the strength of the OPFD
term. To study its effect, we sweep the OPFD weight $\lambda \in \{1, 0.5, 0.1, 0.01\}$ on the OPD-initialized student (Tab.~\ref{tab:lambda_ablation}).
Performance collapses at $\lambda=0.01$, indicating that sufficient OPFD supervision is needed to anchor the self-consistency term, and remains stable within $[0.1, 0.5]$ with a mild trade-off between few- and multi-step performance. 

\textbf{Action execution steps.} The policy executes the first $K$ actions of each predicted chunk before replanning, trading responsiveness for inference overhead. Short execution horizons
improve responsiveness through more frequent replanning, while longer horizons reduce inference overhead at the cost of extended open-loop execution. As shown in Fig.~\ref{fig:action-steps}, we observe that both the teacher $\pi_{0.5}$ and the student perform best at $K=10$, suggesting that the student inherits the teacher's preferred replanning frequency.

\section{Conclusion}
In this paper, we present FastOPD, an efficient on-policy distillation framework that transfers large-scale VLAs into lightweight student models to maximize computational efficiency without compromising performance.
Specifically, FastOPD streamlines the iterative denoising process by anchoring a single on-policy state to the teacher, coupled with a self-consistency objective for finite-interval transitions.
In simulation and real-robot experiments across diverse VLA and WAM teachers, FastOPD improves over the base student in the few-step regime and outperforms existing few-step distillation methods, while substantially reducing training and inference cost. We hope this framework facilitates scalable, real-time deployment of foundation policies on physical robots.

\bibliography{iclr2027_conference}
\bibliographystyle{iclr2027_conference}

\appendix
\section{Proofs}
\label{app:proof}

\opfdprop*
\begin{proof}
For the linear interpolant
\[
\vx_t=(1-t)\vx_0+t\vx_1,
\qquad
\vx_0\sim\mathcal D,
\qquad
\vx_1\sim\mathcal N(\mathbf 0,\mathbf I),
\]
the score and the corresponding instantaneous velocity field satisfy
\begin{equation}
\nabla_{\vx}\log p_t(\vx)
=
-\frac{1}{t}
\left(
\vx+(1-t)\vv(\vx,t)
\right).
\label{eq:score-velocity-relation}
\end{equation}
Following the flow-to-SDE conversion of \citet{liu2025flow}, for any diffusion coefficient $\sigma_t>0$, the
reverse-time SDE that shares the same marginal $p_t$ of PF-ODE is defined as
\begin{equation}
d\vx_t
=
\left(
\vv(\vx_t,t)
-\frac{\sigma_t^2}{2}
\nabla_{\vx}\log p_t(\vx_t)
\right)dt
+
\sigma_t\,d\bar{\mathbf w}_t .
\label{eq:app-reverse-sde}
\end{equation}
Substituting \eqref{eq:score-velocity-relation} into
\eqref{eq:app-reverse-sde} gives
\begin{equation}
d\vx_t
=
\left[
\left(
1+\frac{\sigma_t^2(1-t)}{2t}
\right)\vv(\vx_t,t)
+
\frac{\sigma_t^2}{2t}\vx_t
\right]dt
+
\sigma_t\,d\bar{\mathbf w}_t .
\label{eq:app-reverse-sde-closed}
\end{equation}

A single Euler--Maruyama step of size $\Delta t$ from
$\tilde{\vx}_t$ therefore defines the Gaussian transition kernel
\begin{equation}
p_{\vv}(\vx_{t+\Delta t}\mid\tilde{\vx}_t)
=
\mathcal N
\left(
\mu_{\vv},
\sigma_t^2|\Delta t|\,\mathbf I_d
\right),
\label{eq:app-kernel}
\end{equation}
where
\begin{equation}
\mu_{\vv}
=
\left(
1+\frac{\sigma_t^2}{2t}\Delta t
\right)\tilde{\vx}_t
+
\left(
1+\frac{\sigma_t^2(1-t)}{2t}
\right)
\vv(\tilde{\vx}_t,t)\Delta t .
\end{equation}

For the flow-map parameterization $\vf_\theta(\vx,s,t)=\vx+(t-s)\vu_\theta(\vx,s,t)$,
its tangent velocity is
\begin{equation}
\left.
\partial_\tau\vf_\theta(\vx,t,\tau)
\right|_{\tau=t}
=
\vu_\theta(\vx,t,t).
\label{eq:flow-map-tangent}
\end{equation}
We consequently define the student and teacher velocities at
$\tilde{\vx}_t$ by
\[
\vv_S
:=
\vu_\theta(\tilde{\vx}_t,t,t),
\qquad
\vv_T
:=
\vv_\phi(\tilde{\vx}_t,t).
\]
Applying \eqref{eq:app-kernel} to these two velocities gives
\[
p_\theta(\cdot\mid\tilde{\vx}_t)
=
\mathcal N
\left(
\mu_S,\sigma_t^2|\Delta t|\,\mathbf I_d
\right),
\qquad
p_\phi(\cdot\mid\tilde{\vx}_t)
=
\mathcal N
\left(
\mu_T,\sigma_t^2|\Delta t|\,\mathbf I_d
\right).
\]
Their covariances are identical, and the common
$\tilde{\vx}_t$-dependent term cancels from the difference of their
means. Hence,
\begin{equation}
\mu_S-\mu_T
=
\left(
1+\frac{\sigma_t^2(1-t)}{2t}
\right)
\Delta t\,
(\vv_S-\vv_T).
\end{equation}

Using the KL divergence between Gaussians with common covariance,
we obtain
\begin{align}
\KL\left(
p_\theta(\cdot\mid\tilde{\vx}_t)
\,\middle\|\,
p_\phi(\cdot\mid\tilde{\vx}_t)
\right)
&=
\frac{\|\mu_S-\mu_T\|_2^2}
     {2\sigma_t^2|\Delta t|}
\\
&=
\frac{|\Delta t|}{2\sigma_t^2}
\left(
1+\frac{\sigma_t^2(1-t)}{2t}
\right)^2
\|\vv_S-\vv_T\|_2^2
\\
&=
w_\sigma(t,\Delta t)
\left\|
\vu_\theta(\tilde{\vx}_t,t,t)
-
\vv_\phi(\tilde{\vx}_t,t)
\right\|_2^2,
\end{align}
where
\begin{equation}
w_\sigma(t,\Delta t)
:=
\frac{|\Delta t|}{2\sigma_t^2}
\left(
1+\frac{\sigma_t^2(1-t)}{2t}
\right)^2.
\end{equation}
\end{proof}

\boundprop*
\begin{proof}
Let $(s,t,\vx)\sim q_\theta$ with $\vx=\vx_s=(1-s)\vx_0+s\vx_1$ and $m=(s+t)/2$. Since $s>t$, we have $|m-s|=(s-t)/2\le 1/2$.

We define the self-consistency target of the student as
\begin{equation}
    T_\theta(\vx, s, t) = \frac{1}{2} \big[ \vv_\theta(\vx, s) + \vv_\theta(\hat{\vx}_m, m)\big],
\end{equation}
where $\hat{\vx}_m = \vx + (m-s)\,\vv_\theta(\vx,s)$, and denote the velocity errors at $s$ and $m$ by
\begin{equation}
    e_s := \|\vv_\theta(\vx,s)-\vv_\phi(\vx,s)\|_2,
    \qquad
    e_m := \|\vv_\theta(\hat{\vx}_m,m)-\vv_\phi(\hat{\vx}_m,m)\|_2 .
\end{equation}

Using the definition of $\vu_\phi$, we decompose the difference between the student map and the ideal shortcut as
\begin{align}
    \vu_\theta(\vx, s, t) - \vu_\phi(\vx, s, t)
    &= \big(\vu_\theta(\vx, s, t) - T_\theta\big) + \big(T_\theta - \vu_\phi(\vx, s, t)\big) \nonumber \\
    &= \big(\vu_\theta(\vx, s, t) - T_\theta\big)
    + \tfrac{1}{2}\big[ \vv_\theta(\vx, s) - \vv_\phi(\vx, s)\big] \nonumber \\
    &\quad + \tfrac{1}{2}\big[ \vv_\theta(\hat{\vx}_m, m) - \vv_\phi(\hat{\vx}_m, m)\big]
    + \tfrac{1}{2}\big[ \vv_\phi(\hat{\vx}_m, m) - \vv_\phi(\bar\vx_m, m)\big],
    \label{eq:decomp}
\end{align}
where the first term is the self-consistency error, the second and third terms are the velocity errors at $s$ and $m$, respectively, and the last term is the propagated state error induced by the mismatch between the student midpoint $\hat{\vx}_m$ and the ideal midpoint $\bar{\vx}_m$.

By the $L$-Lipschitzness of the teacher,
\begin{equation}
    \| \vv_\phi(\hat{\vx}_m, m) - \vv_\phi(\bar\vx_m, m) \|_2
    \leq L \| \hat{\vx}_m - \bar\vx_m \|_2
    = L\,|m-s|\, e_s .
\end{equation}

Therefore,
\begin{align}
    \|\vu_\theta(\vx, s, t) - \vu_\phi(\vx, s, t) \|_2^2
    &\leq 4 \| \vu_\theta(\vx, s, t) - T_\theta \|_2^2
    + e_s^2 + e_m^2 + L^2 (m-s)^2 e_s^2 \nonumber \\
    &\leq 4 \| \vu_\theta(\vx, s, t) - T_\theta \|_2^2
    + \Big(1 + \frac{L^2}{4}\Big) e_s^2
    + e_m^2 ,
    \label{eq:pointwise}
\end{align}
where the last inequality uses $|m-s|\le 1/2$.

By taking the expectation of \eqref{eq:pointwise} over $q_\theta$, the left-hand side becomes $D_{q_\theta}(\vu_\theta,\vu_\phi)$ and the first term on the right-hand side becomes $4\mathcal{L}_{\text{SC}}$.
The remaining two terms are velocity errors evaluated at $(s,\vx_s)$ and $(m,\hat{\vx}_m)$, whereas $\mathcal{L}_{\text{OPFD}}$ averages the same error over $(t,\tilde{\vx}_t)\sim\rho_\theta$. 

Let $g(t,\vx):=\|\vv_\theta(\vx,t)-\vv_\phi(\vx,t)\|_2^2\ge 0$, so that $\mathcal{L}_{\text{OPFD}}=\mathbb{E}_{\rho_\theta}[g]$. 

Then, for any distribution $\nu$ whose density ratio with respect to $\rho_\theta$ is at most $\kappa$,
\begin{equation}
    \mathbb{E}_{\nu}[g] = \mathbb{E}_{\rho_\theta}\Big[g\,\frac{d\nu}{d\rho_\theta}\Big] \le \kappa\,\mathbb{E}_{\rho_\theta}[g] = \kappa\,\mathcal{L}_{\text{OPFD}} .
\end{equation}
Applying this with $\nu$ the distributions of $(s,\vx_s)$ and $(m,\hat{\vx}_m)$, which satisfy this condition by the density-ratio assumption, gives $\mathbb{E}_{q_\theta}[e_s^2]\le\kappa\mathcal{L}_{\text{OPFD}}$ and $\mathbb{E}_{q_\theta}[e_m^2]\le\kappa\mathcal{L}_{\text{OPFD}}$. 

Therefore,
\begin{equation}
\begin{aligned}
D_{q_\theta}(\vu_\theta,\vu_\phi)
&\le 4\,\mathcal{L}_{\text{SC}}
  + \kappa\Big(1+\frac{L^2}{4}\Big)\mathcal{L}_{\text{OPFD}}
  + \kappa\,\mathcal{L}_{\text{OPFD}} \\
&= 4\,\mathcal{L}_{\text{SC}}
  + \kappa\Big(2+\frac{L^2}{4}\Big)\mathcal{L}_{\text{OPFD}} \\
&\le \max\Big\{4,\,
  \frac{\kappa(2+L^2/4)}{\lambda}\Big\}
  \,\mathcal{L}_{\text{FastOPD}} \\
&= C_\lambda\,\mathcal{L}_{\text{FastOPD}}.
\end{aligned}
\end{equation}

Here, $C_\lambda$ grows as $\lambda\to 0$, reflecting that the self-consistency term alone cannot anchor the student to the teacher.

\end{proof}

\main*
\begin{proof}
Let the student and the teacher-induced shortcut model start from the same initial noise sample, $\vx_{t_0}^\theta = \bar\vx_{t_0} \sim \mathcal{N}(\mathbf{0}, \mathbf I)$, and follow the same sampling schedule
\begin{equation}
    1 = t_0 > t_1 > \cdots > t_K = 0
\end{equation}
with step sizes $h_k := t_k - t_{k+1} > 0$.

Their rollout states are defined by
\begin{equation}
\vx_{t_{k+1}}^\theta = \vf_\theta(\vx_{t_k}^\theta, t_k, t_{k+1})
\quad \text{and} \quad
\bar\vx_{t_{k+1}} = \vf_{\phi}(\bar\vx_{t_k}, t_k, t_{k+1}),
\end{equation}
 respectively. 
 
By the definition of $\vu_\phi$ and $m_k=(t_k+t_{k+1})/2$,
\begin{equation}
    \vf_{\phi}(\vx,t_k,t_{k+1}) = \vx - h_k\vu_\phi(\vx,t_k,t_{k+1})
    = \bar\vx_{m_k} - \tfrac{h_k}{2}\,\vv_\phi(\bar\vx_{m_k},m_k),  
\end{equation}
where $\bar\vx_{m_k} = \vx - \tfrac{h_k}{2}\,\vv_\phi(\vx,t_k)$.

Thus, $\vf_{\phi}(\cdot,t_k,t_{k+1})$ is the composition of two Euler steps of size $h_k/2$.
Since $\vv_\phi(\cdot,t)$ is $L$-Lipschitz, each Euler step is $(1+h_kL/2)$-Lipschitz, and hence $\vf_{\phi}(\cdot,t_k,t_{k+1})$ is $L_{F,k}$-Lipschitz with

\begin{equation}
    L_{F,k} := \Big(1+\frac{h_k L}{2}\Big)^2 .
\end{equation}
 
Then,
\begin{align}
    \| \vx_{t_{k+1}}^\theta - \bar\vx_{t_{k+1}}\|_2
    &= \| \vf_\theta(\vx_{t_k}^\theta, t_k, t_{k+1}) - \vf_{\phi}(\bar{\vx}_{t_k}, t_k, t_{k+1}) \|_2 \nonumber \\
    &\leq \|\vf_\theta(\vx_{t_k}^\theta, t_k, t_{k+1}) - \vf_{\phi}(\vx_{t_k}^\theta, t_k, t_{k+1})\|_2 \nonumber \\
    &\quad + \|\vf_{\phi}(\vx_{t_k}^\theta, t_k, t_{k+1}) - \vf_{\phi}(\bar\vx_{t_k}, t_k, t_{k+1})\|_2 \nonumber \\
    &\leq h_k \| \vu_\theta( \vx_{t_k}^\theta, t_k, t_{k+1}) - \vu_\phi(\vx_{t_k}^\theta, t_k, t_{k+1} ) \|_2 + L_{F,k} \|\vx_{t_k}^\theta -  \bar\vx_{t_k}\|_2.
\end{align}

Denote $E_{k} := \mathbb{E}\big[ \| \vx_{t_k}^\theta - \bar\vx_{t_k} \|_2^2 \big]$ and $D_k$ as in Theorem~\ref{theorem:main}. 
Then we obtain
\begin{equation}
    E_{k+1} \leq 2h_k^2 D_k + 2L_{F,k}^2 E_k.
    \label{eqn:recursion_2}
\end{equation}
 
Since the two trajectories start from the same noise sample, $E_0=0$. Unrolling \eqref{eqn:recursion_2} gives
\begin{equation}
    E_K \leq \sum_{k=0}^{K-1} a_k D_k,
    \qquad
    a_k := 2h_k^2\prod_{j=k+1}^{K-1} 2L_{F,j}^2.
\end{equation}
By the assumption of $\kappa_{\mathrm{roll}}$ and Proposition~\ref{prop:prop2},
\begin{equation}
    E_K \leq \kappa_{\mathrm{roll}}\Big(\sum_{k=0}^{K-1} a_k\Big) D_{q_\theta}(\vu_\theta,\vu_\phi)
    \leq \kappa_{\mathrm{roll}}\Big(\sum_{k=0}^{K-1} a_k\Big) C_\lambda\, \mathcal{L}_{\text{FastOPD}} .
    \label{eqn:bound1}
\end{equation}
 
The pair $(\vx_{t_K}^\theta, \bar\vx_{t_K})$ constructed from the same initial noise is a valid coupling of $p_\theta$ and $p_{\phi}$. Therefore,
\begin{equation}
    W_2^2(p_\theta, p_{\phi}) \leq \mathbb{E} \big[ \| \vx_{t_K}^\theta - \bar\vx_{t_K} \|^2_2\big] = E_K
    \leq C^2\, \mathcal{L}_{\text{FastOPD}},
    \qquad
    C^2 := \kappa_{\mathrm{roll}}\, C_\lambda \sum_{k=0}^{K-1} a_k .
\end{equation}
Taking square roots on both sides yields 
\begin{equation}
    W_2(p_\theta, p_{\phi}) \leq C \sqrt{\mathcal{L}_{\text{FastOPD}}},
\end{equation}
hence $\mathcal{L}_{\text{FastOPD}} \to 0$ implies $W_2(p_\theta,p_{\phi}) \to 0$.
\end{proof}

\section{Experimental Details}
\label{app:exp-details}

\subsection{Implementation Details.}
All experiments are implemented within the LeRobot framework~\citep{cadene2024lerobot}.
We report the public checkpoints used for teacher and student initialization in Tab.~\ref{tab:model_checkpoints}. For two-time projection, the source and target timesteps are embedded using fixed sinusoidal positional encodings, and the remaining weights of the time projection layer $\psi(\cdot)$ are randomly initialized.

\begin{table}[H]
\centering
\caption{
Model checkpoints used in the experiments.
}
\label{tab:model_checkpoints}
\small
\renewcommand{\arraystretch}{1.0}
\setlength{\tabcolsep}{5pt}

\begin{tabularx}{
    \linewidth
}{
    @{}
    l
    l
    >{\raggedright\arraybackslash}X
    @{}
}
\toprule
\textbf{Benchmark}
& \textbf{Model}
& \textbf{Checkpoint} \\
\midrule

\multirow{3}{*}{LIBERO}
& $\pi_{0.5}$
& \nolinkurl{lerobot/pi05_libero_finetuned_v044} \\

& SmolVLA 
& \nolinkurl{HuggingFaceVLA/smolvla_libero} \\

& SmolVLA (OPD Init)
& \nolinkurl{lerobot/smolvla_base} \\
\midrule

\multirow{3}{*}{RoboTwin 2.0}
& LingBot-VLA
& \nolinkurl{robbyant/lingbot-vla-4b-posttrain-robotwin} \\

& Fast-WAM
& \nolinkurl{lerobot/fastwam_robotwin_uncond_3cam_384} \\

& SmolVLA
& \nolinkurl{lerobot/smolvla_robotwin} \\

\midrule
Real-world
& MolmoAct2
& \nolinkurl{allenai/MolmoAct2-BimanualYAM}  \\

& SmolVLA
& \nolinkurl{lerobot/smolvla_base} \\

\bottomrule
\end{tabularx}
\end{table}

\subsection{Configuration Details.} 
In Tab.~\ref{tab:configurations}, we report training and evaluation configurations for full reproducibility.

\begin{table}[H]
\centering
\caption{
Training and evaluation configurations for FastOPD.
}
\label{tab:configurations}

\scriptsize
\setlength{\tabcolsep}{4pt}
\renewcommand{\arraystretch}{1.1}
\resizebox{\linewidth}{!}{%
\begin{tabular}{lcccc}
\toprule
& \textbf{LIBERO} & \multicolumn{2}{c}{\textbf{RoboTwin~2.0}} & \textbf{Real-world} \\
\midrule
\multicolumn{5}{l}{\textit{Model}} \\
\midrule
Teacher (Params)
    & $\pi_{0.5}$ (3.3B)
    & LingBot-VLA (4.2B)
    & Fast-WAM (6B)
    & MolmoAct2 (5B) \\
Student (Params)
    & SmolVLA + $\psi$ (451M)
    & SmolVLA + $\psi$ (451M)
    & SmolVLA + $\psi$ (451M)
    & SmolVLA + $\psi$ (451M) \\

\midrule
\multicolumn{5}{l}{\textit{Optimization}} \\
\midrule
Optimizer
    & AdamW & AdamW & AdamW & AdamW \\
Adam $(\beta_1, \beta_2)$
    & $(0.9, 0.95)$ & $(0.9, 0.95)$ & $(0.9, 0.95)$ & $(0.9, 0.95)$ \\
Adam $\epsilon$
    & $10^{-8}$ & $10^{-8}$ & $10^{-8}$ & $10^{-8}$ \\
Adam weight decay
    & $10^{-10}$ & $10^{-10}$ & $10^{-10}$ & $10^{-10}$ \\
Gradient clipping
    & 10.0 & 10.0 & 10.0 & 10.0 \\
Scheduler
    & Cosine & Cosine & Cosine & Cosine \\
Learning rate
    & $1\times10^{-4}$ & $1\times10^{-4}$ & $1\times10^{-4}$ & $2\times10^{-5}$ \\

\midrule
\multicolumn{5}{l}{\textit{FastOPD training}} \\
\midrule
Training iterations
    & 10K & 50K & 50K & 20K \\
Batch size
    & 64 & 128 & 128 & 64 \\
$\lambda$
    & 0.1 & 0.1 & 0.1 & 1.0 \\
Time sampling distribution $p$
    & Beta(1.5,1)
    & Beta(1.5,1)
    & Shifted uniform ($s=5$)\footnotemark
    & Beta(1.5,1) \\

\midrule
\multicolumn{5}{l}{\textit{Evaluation}} \\
\midrule
Time discretization
    & Uniform & Uniform & Uniform & Uniform \\
Action execution steps $K$
    & 10 & 50 & 24 & 30 \\

\bottomrule
\end{tabular}%
}
\end{table}
\footnotetext{
We sample $u \sim \mathcal{U}(0,1)$ and apply the flow-shift transformation
$t = \frac{s u}{1 + (s-1)u}$ with $s=5$, following Fast-WAM~\citep{yuan2026fast} implementation.
}

\subsection{Computational Resources}
For full reproducibility, we report the computational resources used in our experiments in Tab.~\ref{tab:hardware}.
Latency is measured on a single RTX~3090
(App.~\ref{app:exp-details}), and real-world policies are deployed on a single RTX~4090
(App.~\ref{app:real-setting}).

\begin{table}[h]
\centering
\caption{Hardware used for training.}
\label{tab:hardware}
\small
\setlength{\tabcolsep}{6pt}
\renewcommand{\arraystretch}{1.0}
\begin{tabular}{ll}
\toprule
\textbf{Experiment} & \textbf{Hardware} \\
\midrule
LIBERO (main)
& $1 \times$ NVIDIA RTX 4090 (24\,GB) \\

LIBERO (ablations)
& $1 \times$ NVIDIA RTX 4090 or RTX 3090 (24\,GB each) \\

RoboTwin 2.0 (LingBot-VLA)
& $2 \times$ NVIDIA RTX 4090 (24\,GB each) \\

RoboTwin 2.0 (Fast-WAM)
& $1 \times$ NVIDIA H200 (141\,GB) \\

Real-world
& $1 \times$ NVIDIA B200 (180\,GB) \\
\bottomrule
\end{tabular}
\end{table}

\subsection{Simulation Experimental Details.}
Evaluation uses 50 episodes per task (2,000 episodes on LIBERO; 2,500 per setting on RoboTwin 2.0), and all results are from a single training run (seed 100000). 
Following the latency details of \cite{luan2026snapflow}, we measure end-to-end inference latency on a single NVIDIA RTX~3090 (CUDA~12.8, PyTorch~2.11) in bfloat16 with batch size 1, excluding pre- and post-processing, and report the median of 100 runs after 10 warm-up runs.\footnote{Fast-WAM does not fit in 24\,GB memory with its text encoder, so we time the encoder separately and add it.}

\subsection{Real-world Experimental Details}
\label{app:real-setting}

\textbf{Robot setup.} 
We follow the real-robot setup of MolmoAct2~\citep{fang2026molmoact2actionreasoningmodels}. Specifically, we use the right arm of a bimanual YAM robot with six joints and a parallel-jaw
gripper, controlled by absolute joint-position commands (7-dimensional action space), with the
joint velocity capped at 0.75\,rad/s (Fig.~\ref{app:robot-setup}).
The policy observes two RGB views (top and right wrist; $480\times640$, 30\,fps; current frame
only) together with the 7-dimensional proprioceptive state (six joints and gripper), and receives
the language instruction.
Actions and states are normalized with quantile normalization for all models (the gripper is not normalized for MolmoAct2). All policies are run on a single NVIDIA RTX~4090, with SmolVLA-based
models in float32.

\begin{figure}[t]
    \centering
    \includegraphics[width=\linewidth]{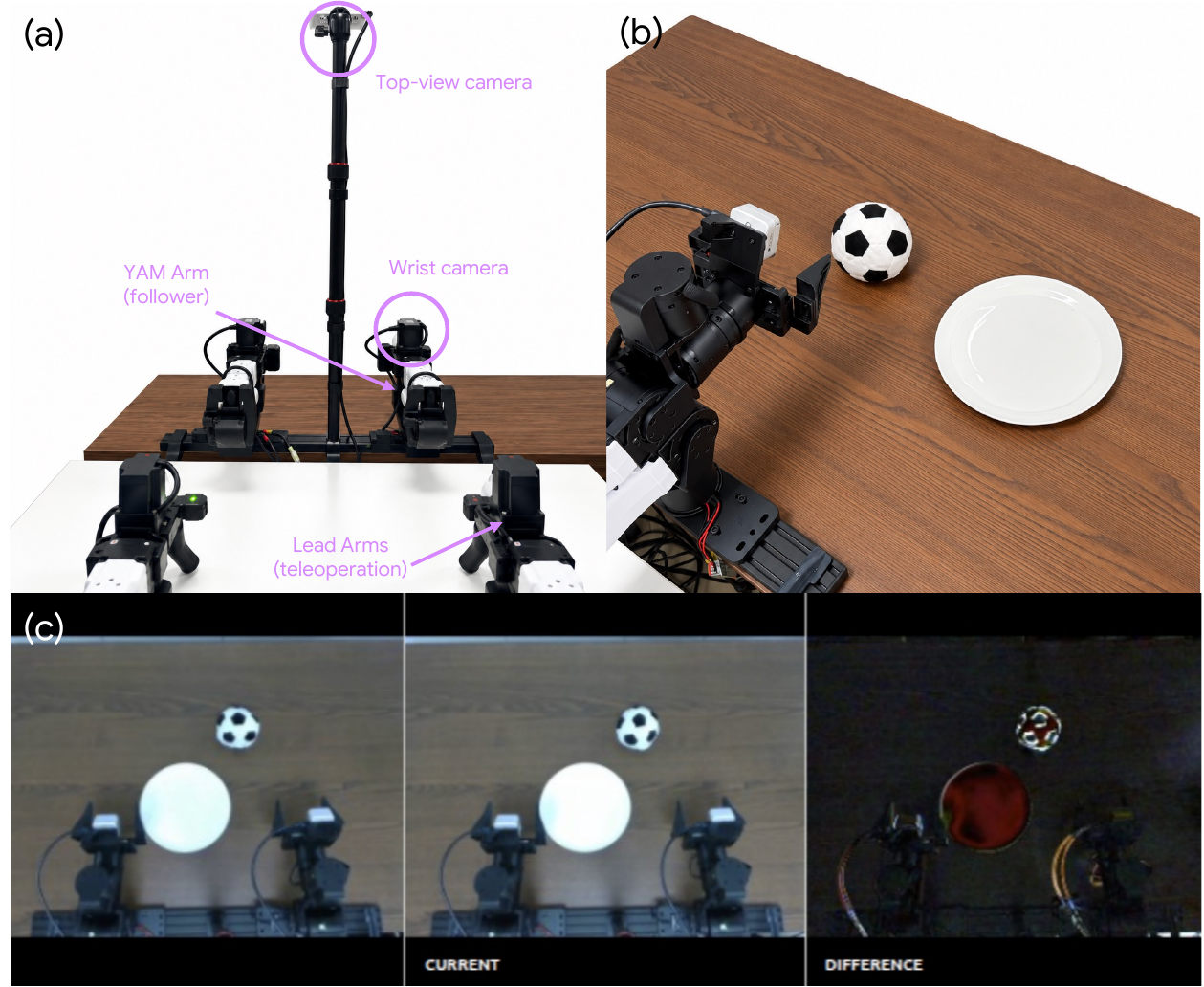}
\caption{\textbf{Robot setting.}
(a) Teleoperation setup for data collection, with lead arms, the follower YAM arm, and top-view and wrist cameras.
(b) Close-up view of the evaluation setup.
(c) Object-position initialization: the current top-camera frame is compared against a reference frame, and their pixel-wise difference is used to place objects at the same initial positions across trials.}
    \label{app:robot-setup}
\end{figure}

\textbf{Dataset.} 
We collect teleoperated demonstrations of multiple tasks on the YAM arm and train all policies
on the same dataset (Tab.~\ref{tab:real_world_dataset}). The evaluation is performed on \texttt{pnp-plate}. The teacher MolmoAct2 is fine-tuned on this dataset for 14{,}000 iterations.
The student is initialized from SmolVLA, warmed up with flow-matching-based behavior cloning for
10{,}000 iterations, and then trained with FastOPD for 20{,}000 iterations.
At inference, the type of target objects including a ball and a plate are fixed and the language instruction is fixed as ``Pick up the ball and place it on the plate''.
\begin{table}[h]
\centering
\caption{Real-world dataset statistics. Average episode length is computed at 30\,fps.}
\label{tab:real_world_dataset}
\small
\begin{tabularx}{\linewidth}{@{}l>{\raggedright\arraybackslash}Xccc@{}}
\toprule
\textbf{Task} & \textbf{Language instruction} & \textbf{Episodes} & \textbf{Frames} & \textbf{Avg. Length (s)} \\
\midrule
\texttt{pnp-plate}  & ``Pick up the ball and place it on the plate.''                 & 106 & 47{,}002 & 14.8 \\
\texttt{pnp-box}    & ``Pick up the basketball and place it in the box.''             & 108 & 47{,}616 & 14.7 \\
\multirow{2}{*}{\texttt{open-close}}
       & ``Open the box, put the ball inside, and close it.'' & 53 & 44{,}780 & 28.2 \\
       & ``Open the box, put the basketball inside, and close it.'' & 55 & 46{,}834 & 28.4 \\
\multirow{2}{*}{\texttt{open-place}}
       & ``Open the box, take the ball out, and place it in the pot.'' & 52 & 41{,}637 & 26.7 \\
       & ``Open the box, take the ball out, and place it on the tray.'' & 50 & 40{,}676 & 27.1 \\
\midrule
Overall             &                                                                 & 424 & 268{,}545 & 21.1 \\
\bottomrule
\end{tabularx}
\end{table}
\begin{figure}[t]
    \centering
    \includegraphics[width=\linewidth]{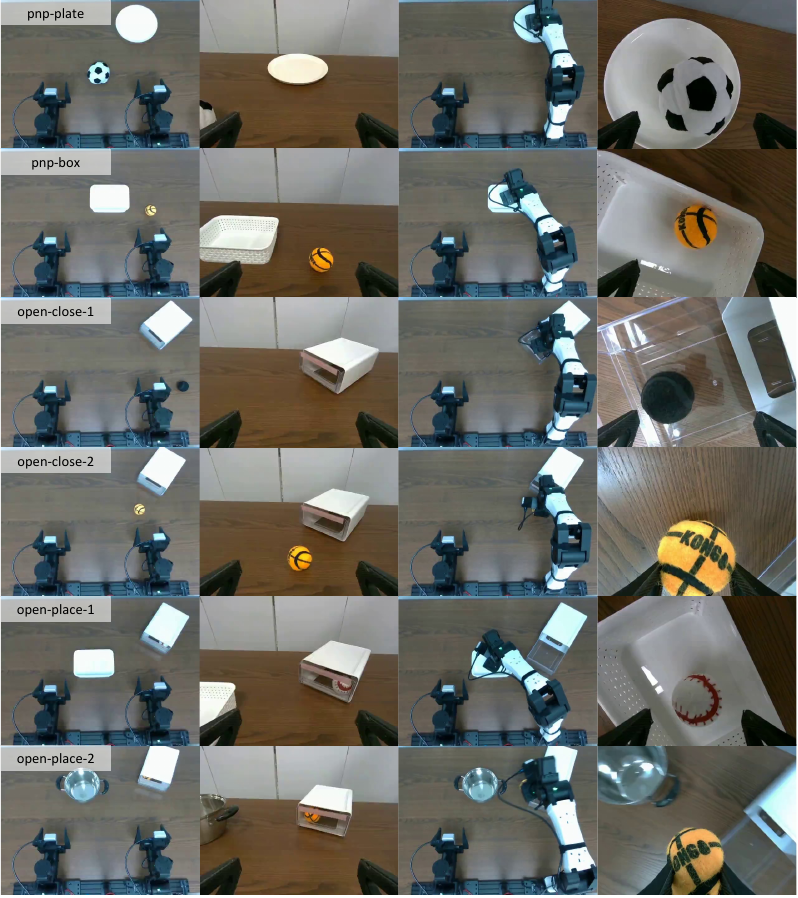}
    \caption{Real-world manipulation task setups on the YAM robot arm. Left two columns show the initial scene and the right show the robot during execution.}
    \label{fig:real-setup}
\end{figure}

\textbf{Policy execution.} Tab.~\ref{tab:real_exec} summarizes the execution settings. All
policies predict 30-step action chunks (1.0\,s) and send commands at 30\,Hz. The SmolVLA (warm-up) and our policy share the same synchronous execution scheme. MolmoAct2 is run with its default asynchronous execution, replanning every 0.5\,s,
skipping the actions already covered by the inference delay, and linearly blending consecutive
chunks over 2--10 steps. No real-time chunking or gripper post-processing is applied to any model.
\begin{table}[h]
\centering
\caption{Real-world policy execution settings. All policies predict 30-step action chunks (1.0\,s)
and send commands at 30\,Hz. For SmolVLA, values are reported for 4 and 10 denoising steps.
Inference latency is measured on the robot side, including HTTP communication.}
\label{tab:real_exec}
\small
\setlength{\tabcolsep}{4pt}
\begin{tabular}{lccc}
\toprule
& \textbf{MolmoAct2} & \textbf{SmolVLA (warm-up)} & \textbf{Ours} \\
\midrule
Execution                  & async (re-plan every 0.5\,s) & sync & sync \\
Sampler                    & Euler, 10 steps & Euler, 4 / 10 steps & flow map, 4 steps \\
Actions executed per chunk & $\sim$15 & 30 & 30 \\
Chunk blending             & linear (2--10 steps) & none & none \\
Inference latency          & 300--450\,ms & 118 / 178\,ms & 120\,ms \\
Re-planning rate           & 1.8--1.9\,Hz & 0.89 / 0.86\,Hz & 0.87\,Hz \\
\bottomrule
\end{tabular}
\end{table}

\textbf{Evaluation details.} Each policy is evaluated over 50 episodes with a time limit of 30\,s. An episode is counted as successful once the ball has been placed entirely within the
plate and the gripper has opened. We report the
success rate and the task completion time, the latter averaged over successful episodes only, following the \cite{wang2024one} evaluation protocol.
To ensure a fair comparison, the initial scene of every episode is matched to the corresponding episode of a reference evaluation session by aligning the current top-camera frame with the
reference image (Fig.~\ref{app:robot-setup}(c)), so that all policies start from identical
object configurations.

\section{Additional Experiments}
\label{app:add-exp}
\textbf{Connection with previous works.}
Although SnapFlow~\citep{luan2026snapflow} adopts a two-step consistency target similar to ours, the two objectives differ in both their anchor and their scope.
SnapFlow is a self-distillation method that combines flow matching on ground-truth actions with a shortcut loss defined only on the fixed interval from $t=1$ to $t=0$:
\begin{equation}
    \mathcal{L}_{\text{SnapFlow}} = \lambda_1\,\mathcal{L}_{\text{FM}} + (1-\lambda_1)\lambda_2\,\mathcal{L}_{\text{shortcut}}, \qquad
    \mathcal{L}_{\text{shortcut}} = \big\|\vu_\theta(\vx_1,1,0\mid c) - \vv_{\text{tgt}}\big\|_2^2,
\end{equation}
where $\vv_{\text{tgt}} = \tfrac{1}{2}\big[\mathrm{sg}[\vv_\theta(\vx_1,1\mid c)] + \mathrm{sg}[\vv_\theta(\hat{\vx}_{0.5},0.5\mid c)]\big]$ and $\hat{\vx}_{0.5} = \vx_1 - 0.5\,\mathrm{sg}[\vv_\theta(\vx_1,1\mid c)]$.
In contrast, FastOPD enforces self-consistency over arbitrary intervals $(s,t)$, yielding a flow map that supports any number of sampling steps, and anchors the target to the teacher through OPFD rather than to the data through flow matching.
SnapFlow can thus be viewed as a teacher-free counterpart of our objective restricted to a single interval, which allows us to isolate the contribution of teacher supervision.

\textbf{Student initialization.}
\label{app:student-init}
In Sec.~\ref{sec:ode-init-ablation}, we studied on-policy distillation (OPD) as an alternative initialization strategy for the flow map.
Specifically, we first train the SmolVLA base model with OPD for 5,000 iterations (Stage~1), and then fine-tune it into a flow map with the FastOPD objective for 10,000 iterations (Stage~2).
Tab.~\ref{tab:init-full} reports the per-suite results, and Fig.~\ref{fig:qual-sr-steps-opd} shows the success rate over training iterations.
As in Fig.~\ref{fig:qual-sr-steps}, FastOPD consistently improves performance over the course of training under this OPD-based initialization as well.

\begin{table}[h]
\centering
\small
\setlength{\tabcolsep}{5pt}
\caption{Per-suite results of OPD-based initialization. Success rates (\%) on LIBERO.}
\label{tab:init-full}
\begin{tabular}{lcccccc}
\toprule
\textbf{Method} & \textbf{NFE} & \textbf{Spatial} & \textbf{Object} & \textbf{Goal} & \textbf{Long} & \textbf{Avg.} \\
\midrule
\multirow[t]{4}{*}{OPD init.}
& 1  & 78.0 & 94.2 & 78.2 & 40.2 & 72.7 \\
& 2  & 84.4 & 95.4 & 87.8 & 45.0 & 78.2 \\
& 4  & 58.0 & 59.8 & 53.8 & 5.0  & 44.2 \\
& 10 & 82.6 & 94.6 & 89.0 & 48.2 & 78.6 \\
\midrule
\multirow[t]{4}{*}{\,+ FastOPD}
& 1  & 84.2 & 92.8 & 90.8 & 54.0 & 80.5 \\
& 2  & 84.6 & 95.2 & 94.2 & 57.8 & \textbf{83.0} \\
& 4  & 83.6 & 97.2 & 91.4 & 56.4 & \underline{82.2} \\
& 10 & 80.4 & 92.8 & 93.4 & 60.2 & 81.7 \\
\bottomrule
\end{tabular}
\end{table}

\begin{figure}[h]
    \centering
    \includegraphics[width=0.4\linewidth]{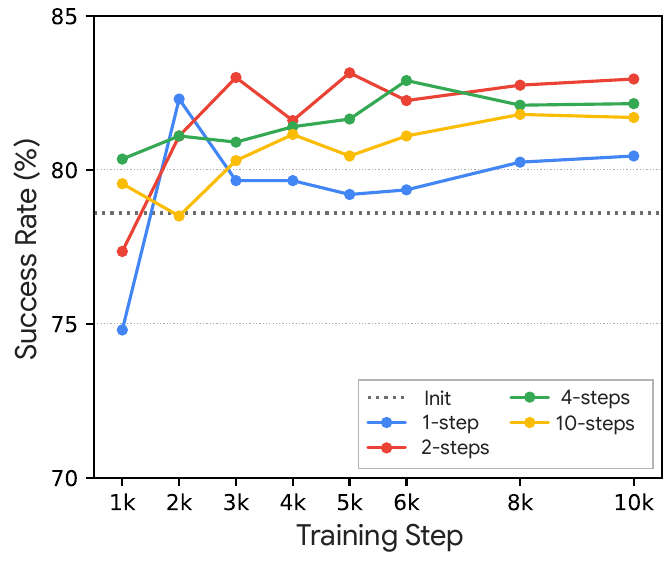}
    \caption{Success rate (\%) on LIBERO over training iterations with OPD-based initialization.}
    \label{fig:qual-sr-steps-opd}
\end{figure}

\section{Few-step Distillation algorithms}
\label{app:alg-details}
In this section, we provide detailed descriptions of the few-step distillation algorithms leveraged in our experiments as baselines.

\subsection{Distribution Matching Distillation}
\label{app:dmd}

Referring to OneDP~\citep{wang2024one}, we design DMD as a one-step generator, and jointly train it with a fake-score critic $\vv_\psi(\vx,t)$
of the same architecture, initialized from the same checkpoint as
$\vv_\theta$. Here, generator and critic are updated alternately at every step.

\textbf{Generator.} A sample is produced with a single Euler step of the
student's own PF-ODE from pure noise $\vx_1\sim\mathcal N(\mathbf 0,\mathbf I)$,
\begin{equation}
    \hat{\vx}_0 = \vx_1 - \vv_\theta(\vx_1,1).
\end{equation}

From the re-noised sample $\vx_t$ from $\hat{\vx}_0$, the teacher and critic denoised estimates are formed as
\begin{equation}
    \hat{\vx}_0^{\mathrm{real}} = \vx_t - t\,\vv_\phi(\vx_t,t), \qquad
    \hat{\vx}_0^{\mathrm{fake}} = \vx_t - t\,\vv_\psi(\vx_t,t).
\end{equation}
The distribution-matching direction is then defined as $\mathbf g_0 = \hat{\vx}_0^{\mathrm{fake}}-\hat{\vx}_0^{\mathrm{real}}$. Here, we normalize it per sample by the mean absolute magnitude of the teacher's own correction,
\begin{equation}
    \mathbf g = \frac{\mathbf g_0}{\mathrm{mean}\big(\big|\hat{\vx}_0^{\mathrm{real}} - \mathrm{sg}[\hat{\vx}_0]\big|\big) + \epsilon},
    \label{eq:dmd-grad}
\end{equation}
where $\mathrm{mean}(\cdot)$ averages over the action-chunk and
action-dimension axes, for training stability. 
The
generator is trained with the surrogate loss whose gradient with respect
to $\hat{\vx}_0$ equals $\mathbf g$:
\begin{equation}
    \mathcal L_{\mathrm{gen}}(\theta) = \tfrac12 \big\|\hat{\vx}_0 - \mathrm{sg}[\hat{\vx}_0 - \mathbf g]\big\|_2^2.
\end{equation}

\textbf{Critic.} The fake-score network $\vv_\psi$ is trained with
ordinary denoising score matching on the generator's output:
\begin{equation}
    \mathcal L_{\mathrm{crit}}(\psi) = \mathbb E_{t', \boldsymbol\epsilon}\Big[\big\|\vv_\psi\big(\vx_{t'} ,t'\big) - \big(\boldsymbol\epsilon - \mathrm{sg}[\hat{\vx}_0]\big)\big\|_2^2\Big],
\end{equation}
where $\vx_{t'}=(1-t')\,\mathrm{sg}[\hat{\vx}_0] + t'\boldsymbol\epsilon$.
\begin{algorithm}[h]
\caption{Training Algorithm for DMD}
\label{alg:dmd}
\begin{algorithmic}[1]
\Require Dataset $\mathcal D$; teacher $\vv_\phi$; generator $\vv_\theta$; critic $\vv_\psi$
\For{each training iteration}
    \State $(\vx_0,c)\sim\mathcal D$, $\vx_1\sim\mathcal N(\mathbf 0,\mathbf I)$
    \State $\hat{\vx}_0 \leftarrow \vx_1 - \vv_\theta(\vx_1,1\mid c)$ \Comment{one-step generation}
    \State $t\sim p(t)$, $\vx_1'\sim\mathcal N(\mathbf 0,\mathbf I)$; \quad $\vx_t \leftarrow (1-t)\,\mathrm{sg}[\hat{\vx}_0] + t\,\vx_1'$
    \State $\hat{\vx}_0^{\mathrm{real}}\leftarrow \vx_t - t\,\vv_\phi(\vx_t,t\mid c)$, \quad $\hat{\vx}_0^{\mathrm{fake}}\leftarrow \vx_t - t\,\vv_\psi(\vx_t,t\mid c)$
    \State $\mathbf g \leftarrow \big(\hat{\vx}_0^{\mathrm{fake}}-\hat{\vx}_0^{\mathrm{real}}\big) \big/ \big(\mathrm{mean}(|\hat{\vx}_0^{\mathrm{real}}-\mathrm{sg}[\hat{\vx}_0]|)+\epsilon\big)$
    \State $\mathcal L_{\mathrm{gen}} \leftarrow \tfrac12\|\hat{\vx}_0 - \mathrm{sg}[\hat{\vx}_0-\mathbf g]\|_2^2$
    \State $\theta \leftarrow \theta - \eta\nabla_\theta \mathcal L_{\mathrm{gen}}$
    \Comment{update generator}
    \State $t'\sim p(t)$, $\boldsymbol\epsilon\sim\mathcal N(\mathbf 0,\mathbf I)$ 
    $\vx_{t'}\leftarrow(1-t')\,\mathrm{sg}[\hat{\vx}_0]+t'\boldsymbol\epsilon$
    \State $\mathcal L_{\mathrm{crit}} \leftarrow \|\vv_\psi(\vx_{t'},t'\mid c) - (\boldsymbol\epsilon - \mathrm{sg}[\hat{\vx}_0])\|_2^2$
    \State $\psi \leftarrow \psi - \eta_\psi\nabla_\psi \mathcal L_{\mathrm{crit}}$
    \Comment{update critic}
\EndFor
\end{algorithmic}
\end{algorithm}

\subsection{Consistency Trajectory Model}
\label{app:ctm}
We follow the Consistency Trajectory Model (CTM) formulation of~\citep{kim2024consistency}, which
generalizes the original Consistency Distillation~\citep{song2023consistency} by allowing anytime-to-anytime jump with a denoising score-matching (DSM) regularizer. 
We train the two-time flow map
$\vf_\theta(\vx,s,t)=\vx+(t-s)\vu_\theta(\vx,s,t)$ together with an EMA target network $\theta^-$,
updated as $\theta^- \leftarrow \mu\theta^- + (1-\mu)\theta$.

For sampled timesteps $1\ge s\ge r\ge0$ and an intermediate solver time $u\sim\mathrm{Unif}[r,s]$,
with $\vx_1\sim\mathcal N(\mathbf 0,\mathbf I)$, $\vx_0\sim p_0$, and
$\vx_s=(1-s)\vx_0+s\vx_1$, the student first jumps from $s$ to $r$,
\begin{equation}
    \hat{\vx}_r^{\theta} = \vf_\theta(\vx_s,s,r),
\end{equation}
while the target integrates the teacher's PF-ODE from $s$ to $u$ with a 1-step Heun solver,
$\vx_u = \mathrm{ODE}_\phi(\vx_s;s,u)$, and then estimates from time $u$ to $r$ with the EMA network,
\begin{equation}
    \vx_r^{-} = \vf_{\theta^-}(\vx_u,u,r).
\end{equation}
Both are then projected to the clean timestep $t=0$ through a stop-gradient student $G_{\mathrm{sg}}(\cdot,r,0):=\vf_{\mathrm{sg}[\theta]}(\cdot,r,0)$, giving
$\hat{\vx}_0^{\theta}=G_{\mathrm{sg}}(\hat{\vx}_r^{\theta},r,0)$ and
$\vx_0^{-}=G_{\mathrm{sg}}(\vx_r^{-},r,0)$.
Thus, the CTM loss is defined as
\begin{equation}
    \mathcal L_{\mathrm{CTM}}(\theta) = \mathbb E\big\|\hat{\vx}_0^{\theta} - \mathrm{sg}[\vx_0^{-}]\big\|_2^2,
\end{equation}
and we add a DSM regularizer anchoring the student's instantaneous velocity $\vv_\theta(\vx_s,s)$
to the teacher's,\footnote{The regularizer can instead target the ground-truth flow velocity
$\vx_1-\vx_0$ (dataset-supervised DSM, as in~\citep{kim2024consistency}); we use the
teacher-supervised variant.}
\begin{equation}
    \mathcal L_{\mathrm{DSM}}(\theta) = \mathbb E\big\|\vv_\theta(\vx_s,s) - \mathrm{sg}[\vv_\phi(\vx_s,s)]\big\|_2^2,
\end{equation}
so the total loss is
$\mathcal L_{\mathrm{}}(\theta) = \mathcal L_{\mathrm{CTM}}(\theta) + w_{\mathrm{dsm}}\,\mathcal L_{\mathrm{DSM}}(\theta)$.
\begin{algorithm}[h]
\caption{Training Algorithm for CTM}
\label{alg:ctm}
\begin{algorithmic}[1]
\Require Dataset $\mathcal D$; teacher $\vv_\phi$; student flow map $\vu_\theta$; EMA target $\theta^-$; decay $\mu$; DSM weight $w_{\mathrm{dsm}}$
\For{each training iteration}
    \State $(\vx_0,c)\sim\mathcal D$, $\vx_1\sim\mathcal N(\mathbf 0,\mathbf I)$, sample $1\ge s\ge r\ge0$, $u\sim\mathrm{Unif}[r,s]$
    \State $\vx_s \leftarrow (1-s)\vx_0+s\vx_1$
    \State $\hat{\vx}_r^\theta \leftarrow \vf_\theta(\vx_s,s,r\mid c)$
    \State $\vx_u \leftarrow \mathrm{ODE}_\phi(\vx_s; s,u \mid c)$
    \State $\vx_r^{-} \leftarrow \vf_{\theta^-}(\vx_u,u,r\mid c)$
    \State $\hat{\vx}_0^\theta \leftarrow \vf_{\mathrm{sg}[\theta]}(\hat{\vx}_r^\theta,r,0\mid c)$;
    $\vx_0^{-} \leftarrow \vf_{\mathrm{sg}[\theta]}(\vx_r^{-},r,0\mid c)$
    \State $\mathcal L_{\mathrm{CTM}} \leftarrow \|\hat{\vx}_0^\theta - \mathrm{sg}[\vx_0^{-}]\|_2^2$
    \State $\mathcal L_{\mathrm{DSM}} \leftarrow \|\vu_\theta(\vx_s,s,s\mid c) - \mathrm{sg}[\vv_\phi(\vx_s,s\mid c)]\|_2^2$
    \State $\theta \leftarrow \theta - \eta\nabla_\theta\big(\mathcal L_{\mathrm{CTM}} + w_{\mathrm{dsm}}\mathcal L_{\mathrm{DSM}}\big)$
    \State $\theta^- \leftarrow \mu\theta^- + (1-\mu)\theta$
\EndFor
\end{algorithmic}
\end{algorithm}

\subsection{improved MeanFlow}
\label{app:imf}

MeanFlow~\citep{geng2025mean} parameterizes the average velocity
$\vu(\vx,s,t):=\frac1{s-t}\int_t^s \vv(\vx_\tau,\tau)\,d\tau$ and train it
through the identity obtained by differentiating the definition along $s$,
\begin{equation}
    \vu(\vx_s,s,t) = \vv(\vx_s,s) - (s-t)\,\frac{d}{ds}\vu(\vx_s,s,t),
    \qquad \frac{d}{ds}\vu = \partial_\vx\vu\cdot\vv(\vx_s,s) + \partial_s\vu.
    \label{eq:meanflow-identity}
\end{equation}
In the original formulation, $\vv(\vx_s,s)$ in
\eqref{eq:meanflow-identity} is taken to be the student's own
instantaneous velocity $\vv_\theta(\vx_s,s):=\vu_\theta(\vx_s,s,s)$, so
the regression target is itself computed by the network being
optimized. 
Improved MeanFlow (iMF)~\citep{geng2026improved} addresses this issue by
reformulating the identity to remove this self-reference. 

We follow this reformulation and further adapt it to distillation by replacing the self-referential term with the frozen teacher velocity,
$\vv_\phi(\vx_s,s)$, used both as the target and
the tangent direction of the total derivative.

Given
$\vx_1\sim\mathcal N(\mathbf 0,\mathbf I)$, $\vx_0\sim p_0$, and interpolant 
$\vx_s$, the total derivative is computed by a
Jacobian--Vector Product (JVP) of $\vu_\theta$ at $(\vx_s,s,t)$ along the
tangent $(\vv_\phi,1,0)$:
\begin{equation}
    \dot{\vu}_\theta := \mathrm{JVP}\big(\vu_\theta;\,(\vx_s,s,t),\,(\vv_\phi,1,0)\big),
\end{equation}
The model then predicts one-step velocity as follow:
\begin{equation}
    \hat{\vv}_\theta = \vu_\theta(\vx_s,s,t) + (s-t)\,\mathrm{sg}[\dot{\vu}_\theta].
\end{equation}
Note that MF loss reduces to standard flow matching when $t=s$. 
The loss regresses $\hat{\vv}_\theta$ onto the teacher target with an
elementwise Pseudo-Huber loss,
\begin{equation}
    \mathcal L_{\mathrm{iMF}}(\theta) = \mathbb E\big[\rho_\delta(\hat{\vv}_\theta - \vv_\phi)\big],
    \qquad \rho_\delta(e) = \delta^2\Big(\sqrt{1+e^2/\delta^2}-1\Big),
\end{equation}
applied element-wise and averaged over the action dimensions.

\begin{algorithm}[h]
\caption{Training Algorithm for iMF}
\label{alg:imf}
\begin{algorithmic}[1]
\Require Dataset $\mathcal D$; teacher $\vv_\phi$; student flow map $\vu_\theta$; FM prob. $p_{t=s}$; Huber scale $\delta$
\For{each training iteration}
    \State $(\vx_0,c)\sim\mathcal D$, $\vx_1\sim\mathcal N(\mathbf 0,\mathbf I)$, sample $1\ge s\ge t\ge0$
    \State With prob. $p_{t=s}$: $t \leftarrow s$ %
    \State $\vx_s \leftarrow (1-s)\vx_0 + s\vx_1$
    \State $\vu_\theta(\vx_s,s,t\mid c),\ \dot{\vu}_\theta \leftarrow \mathrm{sg}\left[\mathrm{JVP}\big(\vu_\theta(\cdot,\cdot,t\mid c);\,(\vx_s,s),\,(\vv_\phi,1)\big)\right]$
    \State $\hat{\vv}_\theta \leftarrow \vu_\theta(\vx_s,s,t\mid c) + (s-t)\,\dot{\vu}_\theta$
    \State $\mathcal L_{\mathrm{iMF}} \leftarrow \rho_\delta(\hat{\vv}_\theta - \vv_\phi)$ 
    \State $\theta \leftarrow \theta - \eta\nabla_\theta \mathcal L_{\mathrm{iMF}}$
\EndFor
\end{algorithmic}
\end{algorithm}

\newpage
\raggedbottom
\section{Detailed Results in Simulation Environments}
\label{app:detailed_results}
\begin{table}[H]
\centering
\scriptsize
\setlength{\tabcolsep}{5.0pt}
\renewcommand{\arraystretch}{1.05}

\caption{
Per-task success rates (\%) on RoboTwin~2.0 at 10 denoising steps, under clean and randomized settings. Ours (X) denotes our method trained with X as the teacher.
}
\label{tab:robotwin_per_task_10step}

\resizebox{\linewidth}{!}{%
\begin{tabular}{l cc cc cc cc cc}
\toprule

\multirow{3}{*}{\textbf{Task}}
& \multicolumn{2}{c}{\textbf{LingBot-VLA}}
& \multicolumn{2}{c}{\textbf{Fast-WAM}}
& \multicolumn{2}{c}{\textbf{SmolVLA}}
& \multicolumn{2}{c}{\textbf{Ours (LingBot-VLA)}}
& \multicolumn{2}{c}{\textbf{Ours (Fast-WAM)}} \\

\cmidrule(lr){2-3}\cmidrule(lr){4-5}\cmidrule(lr){6-7}\cmidrule(lr){8-9}\cmidrule(lr){10-11}

& \textbf{Clean} & \textbf{Random}
& \textbf{Clean} & \textbf{Random}
& \textbf{Clean} & \textbf{Random}
& \textbf{Clean} & \textbf{Random}
& \textbf{Clean} & \textbf{Random} \\

\midrule

adjust\_bottle & 98 & 90 & 100 & 100 & 100 & 98 & 100 & 98 & 100 & 98 \\
beat\_block\_hammer & 78 & 78 & 100 & 96 & 66 & 66 & 56 & 40 & 76 & 78 \\
blocks\_ranking\_rgb & 96 & 92 & 100 & 100 & 50 & 46 & 50 & 58 & 30 & 30 \\
blocks\_ranking\_size & 74 & 66 & 88 & 96 & 20 & 22 & 28 & 26 & 20 & 10 \\
click\_alarmclock & 70 & 42 & 100 & 100 & 90 & 94 & 34 & 50 & 92 & 96 \\
click\_bell & 62 & 78 & 100 & 100 & 96 & 92 & 62 & 54 & 96 & 86 \\
dump\_bin\_bigbin & 92 & 98 & 100 & 96 & 96 & 70 & 92 & 84 & 90 & 82 \\
grab\_roller & 100 & 98 & 100 & 100 & 96 & 96 & 100 & 100 & 100 & 100 \\
handover\_block & 88 & 82 & 96 & 90 & 28 & 22 & 18 & 32 & 28 & 28 \\
handover\_mic & 100 & 88 & 100 & 98 & 80 & 74 & 76 & 82 & 84 & 88 \\
hanging\_mug & 48 & 50 & 64 & 56 & 10 & 16 & 4 & 12 & 6 & 8 \\
lift\_pot & 100 & 100 & 100 & 100 & 84 & 72 & 62 & 68 & 80 & 80 \\
move\_can\_pot & 66 & 86 & 88 & 98 & 22 & 44 & 18 & 44 & 10 & 32 \\
move\_pillbottle\_pad & 98 & 94 & 96 & 100 & 54 & 54 & 50 & 56 & 52 & 56 \\
move\_playingcard\_away & 100 & 100 & 98 & 98 & 78 & 84 & 88 & 88 & 68 & 80 \\
move\_stapler\_pad & 56 & 66 & 74 & 66 & 24 & 14 & 30 & 22 & 26 & 18 \\
open\_laptop & 96 & 94 & 94 & 100 & 88 & 86 & 78 & 78 & 72 & 88 \\
open\_microwave & 58 & 38 & 48 & 46 & 18 & 36 & 20 & 54 & 22 & 30 \\
pick\_diverse\_bottles & 82 & 76 & 82 & 90 & 42 & 32 & 60 & 44 & 50 & 40 \\
pick\_dual\_bottles & 94 & 88 & 100 & 100 & 50 & 42 & 50 & 42 & 64 & 64 \\
place\_a2b\_left & 84 & 80 & 100 & 98 & 38 & 48 & 46 & 42 & 40 & 42 \\
place\_a2b\_right & 76 & 82 & 94 & 98 & 40 & 36 & 38 & 46 & 38 & 44 \\
place\_bread\_basket & 84 & 90 & 94 & 90 & 68 & 68 & 66 & 56 & 60 & 60 \\
place\_bread\_skillet & 92 & 70 & 92 & 92 & 62 & 54 & 66 & 62 & 68 & 70 \\
place\_burger\_fries & 94 & 94 & 94 & 100 & 96 & 84 & 96 & 92 & 98 & 90 \\
place\_can\_basket & 82 & 72 & 76 & 68 & 62 & 44 & 30 & 24 & 32 & 8 \\
place\_cans\_plasticbox & 100 & 100 & 100 & 100 & 82 & 84 & 58 & 52 & 88 & 82 \\
place\_container\_plate & 98 & 96 & 100 & 96 & 90 & 94 & 98 & 96 & 90 & 88 \\
place\_dual\_shoes & 80 & 84 & 94 & 88 & 50 & 62 & 58 & 76 & 42 & 50 \\
place\_empty\_cup & 100 & 98 & 100 & 100 & 96 & 98 & 98 & 98 & 86 & 84 \\
place\_fan & 88 & 80 & 98 & 96 & 38 & 50 & 38 & 44 & 42 & 56 \\
place\_mouse\_pad & 86 & 90 & 82 & 84 & 26 & 42 & 34 & 42 & 30 & 36 \\
place\_object\_basket & 92 & 84 & 88 & 82 & 68 & 62 & 82 & 80 & 70 & 48 \\
place\_object\_scale & 82 & 82 & 100 & 94 & 46 & 38 & 50 & 58 & 36 & 52 \\
place\_object\_stand & 92 & 96 & 98 & 98 & 60 & 82 & 62 & 82 & 70 & 86 \\
place\_phone\_stand & 80 & 96 & 92 & 96 & 48 & 74 & 46 & 76 & 60 & 70 \\
place\_shoe & 96 & 96 & 88 & 92 & 86 & 88 & 90 & 88 & 76 & 78 \\
press\_stapler & 94 & 96 & 90 & 98 & 78 & 74 & 66 & 64 & 64 & 72 \\
put\_bottles\_dustbin & 90 & 84 & 90 & 92 & 36 & 42 & 34 & 44 & 30 & 38 \\
put\_object\_cabinet & 90 & 78 & 88 & 90 & 60 & 60 & 60 & 54 & 72 & 64 \\
rotate\_qrcode & 88 & 88 & 94 & 82 & 64 & 64 & 64 & 68 & 56 & 52 \\
scan\_object & 96 & 94 & 90 & 94 & 60 & 70 & 46 & 48 & 54 & 54 \\
shake\_bottle & 100 & 100 & 100 & 100 & 100 & 98 & 98 & 98 & 100 & 96 \\
shake\_bottle\_horizontally & 100 & 98 & 100 & 100 & 98 & 98 & 98 & 98 & 96 & 98 \\
stack\_blocks\_three & 92 & 94 & 94 & 100 & 4 & 20 & 4 & 6 & 8 & 16 \\
stack\_blocks\_two & 100 & 100 & 100 & 100 & 40 & 50 & 64 & 66 & 42 & 58 \\
stack\_bowls\_three & 86 & 72 & 82 & 76 & 48 & 56 & 36 & 46 & 52 & 50 \\
stack\_bowls\_two & 96 & 98 & 92 & 98 & 70 & 84 & 72 & 74 & 80 & 74 \\
stamp\_seal & 88 & 92 & 90 & 92 & 26 & 34 & 40 & 30 & 18 & 32 \\
turn\_switch & 56 & 56 & 68 & 68 & 40 & 48 & 32 & 36 & 50 & 38 \\

\midrule

\textbf{Average} & 86.76 & 84.88 & 91.92 & 91.84 & 59.44 & 61.32 & 56.92 & 59.56 & 58.28 & 59.56 \\

\bottomrule
\end{tabular}
}
\end{table}

\begin{table}[H]
\centering
\scriptsize
\setlength{\tabcolsep}{3.0pt}
\renewcommand{\arraystretch}{1.1}

\caption{
Per-task success rates (\%) of our method with LingBot-VLA and Fast-WAM as teachers on RoboTwin~2.0 at 1, 2, and 4 denoising steps, under clean and randomized settings. Ours (X) denotes our method trained with X as the teacher.
}
\label{tab:robotwin_per_task_stepwise_ours}

\resizebox{\textwidth}{!}{
\begin{tabular}{l cc cc cc cc cc cc}
\toprule

\multirow{3}{*}{\textbf{Task}}
& \multicolumn{6}{c}{\textbf{Ours (LingBot-VLA)}}
& \multicolumn{6}{c}{\textbf{Ours (Fast-WAM)}} \\

\cmidrule(lr){2-7}\cmidrule(lr){8-13}

& \multicolumn{2}{c}{\textbf{1-step}}
& \multicolumn{2}{c}{\textbf{2-step}}
& \multicolumn{2}{c}{\textbf{4-step}}
& \multicolumn{2}{c}{\textbf{1-step}}
& \multicolumn{2}{c}{\textbf{2-step}}
& \multicolumn{2}{c}{\textbf{4-step}} \\

\cmidrule(lr){2-3}\cmidrule(lr){4-5}\cmidrule(lr){6-7}\cmidrule(lr){8-9}\cmidrule(lr){10-11}\cmidrule(lr){12-13}

& \textbf{Clean} & \textbf{Random}
& \textbf{Clean} & \textbf{Random}
& \textbf{Clean} & \textbf{Random}
& \textbf{Clean} & \textbf{Random}
& \textbf{Clean} & \textbf{Random}
& \textbf{Clean} & \textbf{Random} \\

\midrule

adjust\_bottle & 94 & 96 & 100 & 98 & 100 & 98 & 94 & 98 & 100 & 96 & 100 & 98 \\
beat\_block\_hammer & 72 & 58 & 84 & 68 & 66 & 66 & 56 & 50 & 76 & 68 & 80 & 74 \\
blocks\_ranking\_rgb & 34 & 32 & 38 & 44 & 52 & 56 & 34 & 32 & 44 & 32 & 32 & 32 \\
blocks\_ranking\_size & 6 & 6 & 20 & 18 & 28 & 16 & 12 & 8 & 12 & 10 & 20 & 14 \\
click\_alarmclock & 24 & 46 & 22 & 38 & 36 & 54 & 90 & 82 & 82 & 92 & 90 & 90 \\
click\_bell & 94 & 92 & 82 & 88 & 66 & 74 & 100 & 100 & 98 & 94 & 98 & 92 \\
dump\_bin\_bigbin & 60 & 56 & 76 & 72 & 84 & 78 & 76 & 76 & 70 & 68 & 76 & 86 \\
grab\_roller & 96 & 98 & 96 & 98 & 98 & 96 & 100 & 98 & 100 & 100 & 100 & 98 \\
handover\_block & 12 & 12 & 14 & 16 & 26 & 22 & 20 & 12 & 22 & 24 & 30 & 20 \\
handover\_mic & 58 & 66 & 70 & 64 & 72 & 70 & 58 & 68 & 72 & 74 & 86 & 92 \\
hanging\_mug & 6 & 6 & 10 & 12 & 6 & 18 & 2 & 4 & 4 & 2 & 6 & 10 \\
lift\_pot & 78 & 80 & 92 & 86 & 84 & 86 & 70 & 74 & 74 & 74 & 82 & 76 \\
move\_can\_pot & 38 & 32 & 22 & 46 & 32 & 46 & 22 & 26 & 26 & 40 & 22 & 20 \\
move\_pillbottle\_pad & 50 & 40 & 70 & 70 & 72 & 68 & 44 & 52 & 60 & 62 & 60 & 56 \\
move\_playingcard\_away & 70 & 82 & 86 & 84 & 86 & 86 & 68 & 78 & 72 & 78 & 76 & 84 \\
move\_stapler\_pad & 20 & 22 & 20 & 20 & 24 & 20 & 14 & 16 & 16 & 22 & 16 & 32 \\
open\_laptop & 92 & 94 & 88 & 90 & 82 & 84 & 78 & 86 & 76 & 90 & 76 & 92 \\
open\_microwave & 6 & 12 & 10 & 12 & 12 & 26 & 10 & 4 & 14 & 20 & 14 & 26 \\
pick\_diverse\_bottles & 56 & 36 & 66 & 44 & 58 & 44 & 48 & 40 & 48 & 36 & 50 & 42 \\
pick\_dual\_bottles & 74 & 58 & 60 & 62 & 58 & 52 & 52 & 64 & 70 & 58 & 68 & 68 \\
place\_a2b\_left & 38 & 42 & 30 & 44 & 38 & 44 & 32 & 40 & 34 & 42 & 40 & 42 \\
place\_a2b\_right & 28 & 34 & 36 & 38 & 38 & 38 & 28 & 38 & 46 & 44 & 40 & 50 \\
place\_bread\_basket & 52 & 52 & 58 & 56 & 64 & 60 & 64 & 56 & 70 & 64 & 70 & 64 \\
place\_bread\_skillet & 80 & 64 & 70 & 58 & 76 & 68 & 70 & 52 & 76 & 62 & 68 & 68 \\
place\_burger\_fries & 96 & 90 & 92 & 92 & 98 & 94 & 90 & 90 & 98 & 96 & 98 & 96 \\
place\_can\_basket & 40 & 38 & 58 & 40 & 44 & 36 & 42 & 24 & 34 & 22 & 26 & 16 \\
place\_cans\_plasticbox & 58 & 60 & 82 & 78 & 80 & 76 & 52 & 44 & 72 & 78 & 84 & 78 \\
place\_container\_plate & 84 & 90 & 96 & 88 & 90 & 94 & 88 & 92 & 76 & 82 & 92 & 92 \\
place\_dual\_shoes & 52 & 52 & 48 & 70 & 56 & 68 & 50 & 54 & 56 & 62 & 58 & 56 \\
place\_empty\_cup & 78 & 84 & 94 & 96 & 90 & 98 & 72 & 56 & 70 & 68 & 86 & 76 \\
place\_fan & 16 & 16 & 40 & 48 & 48 & 44 & 24 & 14 & 38 & 52 & 44 & 44 \\
place\_mouse\_pad & 28 & 24 & 26 & 46 & 28 & 50 & 20 & 28 & 30 & 36 & 38 & 26 \\
place\_object\_basket & 66 & 58 & 80 & 72 & 70 & 74 & 44 & 30 & 56 & 32 & 66 & 56 \\
place\_object\_scale & 40 & 46 & 52 & 60 & 54 & 56 & 32 & 34 & 38 & 58 & 44 & 50 \\
place\_object\_stand & 66 & 80 & 66 & 82 & 70 & 86 & 70 & 64 & 70 & 88 & 74 & 90 \\
place\_phone\_stand & 26 & 52 & 44 & 58 & 54 & 72 & 40 & 48 & 40 & 58 & 48 & 62 \\
place\_shoe & 82 & 84 & 84 & 86 & 84 & 90 & 66 & 80 & 74 & 82 & 68 & 80 \\
press\_stapler & 86 & 78 & 74 & 74 & 72 & 74 & 80 & 84 & 74 & 72 & 74 & 70 \\
put\_bottles\_dustbin & 8 & 10 & 20 & 10 & 18 & 34 & 6 & 4 & 14 & 22 & 34 & 38 \\
put\_object\_cabinet & 40 & 28 & 52 & 46 & 56 & 54 & 52 & 46 & 68 & 62 & 58 & 52 \\
rotate\_qrcode & 66 & 54 & 60 & 64 & 60 & 72 & 58 & 52 & 64 & 54 & 52 & 58 \\
scan\_object & 58 & 48 & 58 & 64 & 52 & 64 & 64 & 50 & 54 & 54 & 50 & 62 \\
shake\_bottle & 98 & 98 & 100 & 100 & 96 & 96 & 98 & 96 & 96 & 98 & 98 & 96 \\
shake\_bottle\_horizontally & 96 & 94 & 98 & 98 & 100 & 98 & 96 & 94 & 100 & 98 & 94 & 94 \\
stack\_blocks\_three & 6 & 2 & 2 & 4 & 10 & 14 & 10 & 0 & 8 & 12 & 6 & 10 \\
stack\_blocks\_two & 20 & 16 & 32 & 36 & 46 & 56 & 30 & 36 & 42 & 50 & 56 & 58 \\
stack\_bowls\_three & 6 & 12 & 36 & 30 & 50 & 34 & 16 & 14 & 28 & 24 & 42 & 40 \\
stack\_bowls\_two & 68 & 46 & 74 & 70 & 76 & 86 & 62 & 48 & 58 & 70 & 84 & 68 \\
stamp\_seal & 12 & 20 & 20 & 38 & 30 & 46 & 14 & 18 & 16 & 36 & 10 & 28 \\
turn\_switch & 40 & 48 & 46 & 38 & 40 & 26 & 38 & 36 & 40 & 40 & 40 & 38 \\

\midrule

\textbf{Average} & 51.48 & 50.88 & 57.08 & 58.28 & 58.60 & 61.24 & 51.12 & 49.80 & 55.52 & 57.16 & 58.48 & 59.20 \\

\bottomrule
\end{tabular}
}

\end{table}

\newpage
\section{Qualitative Results}
\label{app:qualitative}
We provide qualitative results in simulation (Fig.~\ref{fig:robotwin-qual}) and on the real-robot (Fig.~\ref{fig:real-qual}) experiments.

\begin{figure}[h]
    \centering
    \includegraphics[width=\linewidth]{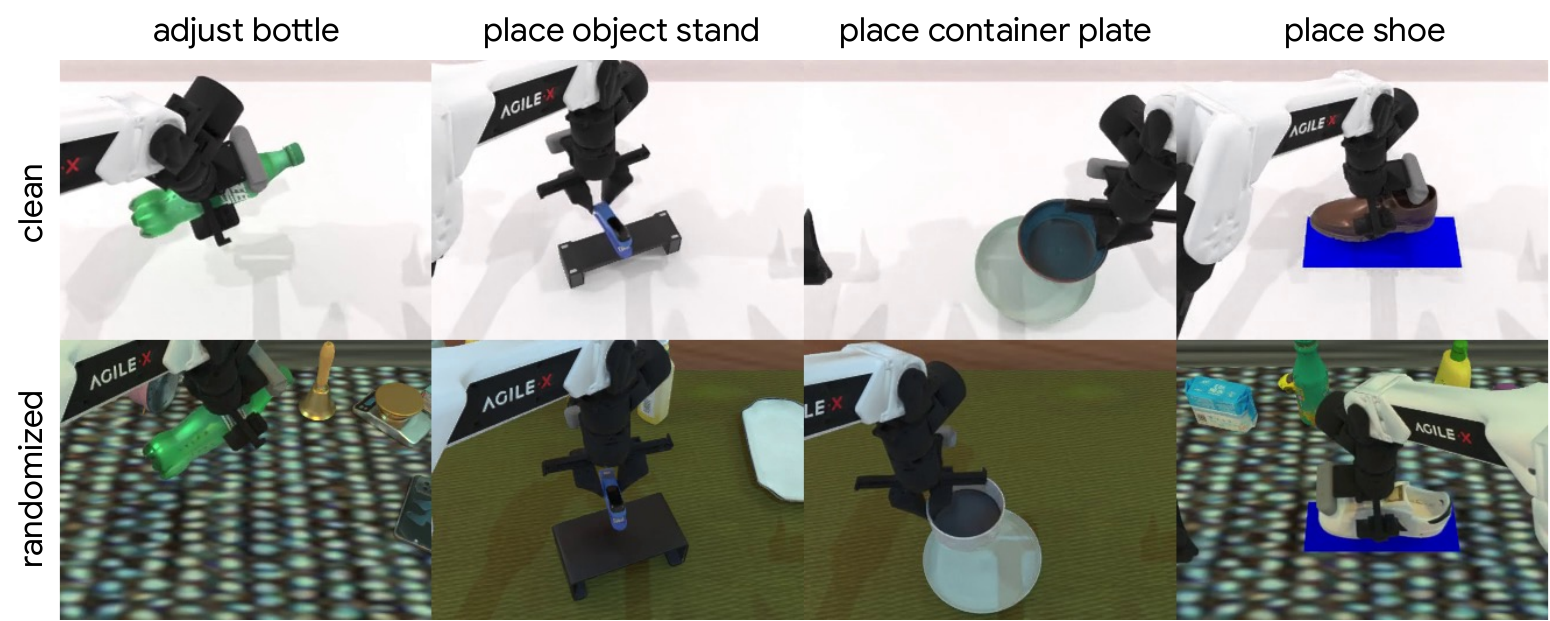}
    \caption{\textbf{Qualitative results on RoboTwin~2.0.} FastOPD distilled from LingBot-VLA (NFE=1) executing four tasks under the clean (top) and randomized (bottom) settings, where the randomized setting introduces cluttered distractors, textured tabletops, and varied backgrounds.}
    \label{fig:robotwin-qual}
\end{figure}

\begin{figure}[h]
    \centering
    \includegraphics[width=\linewidth]{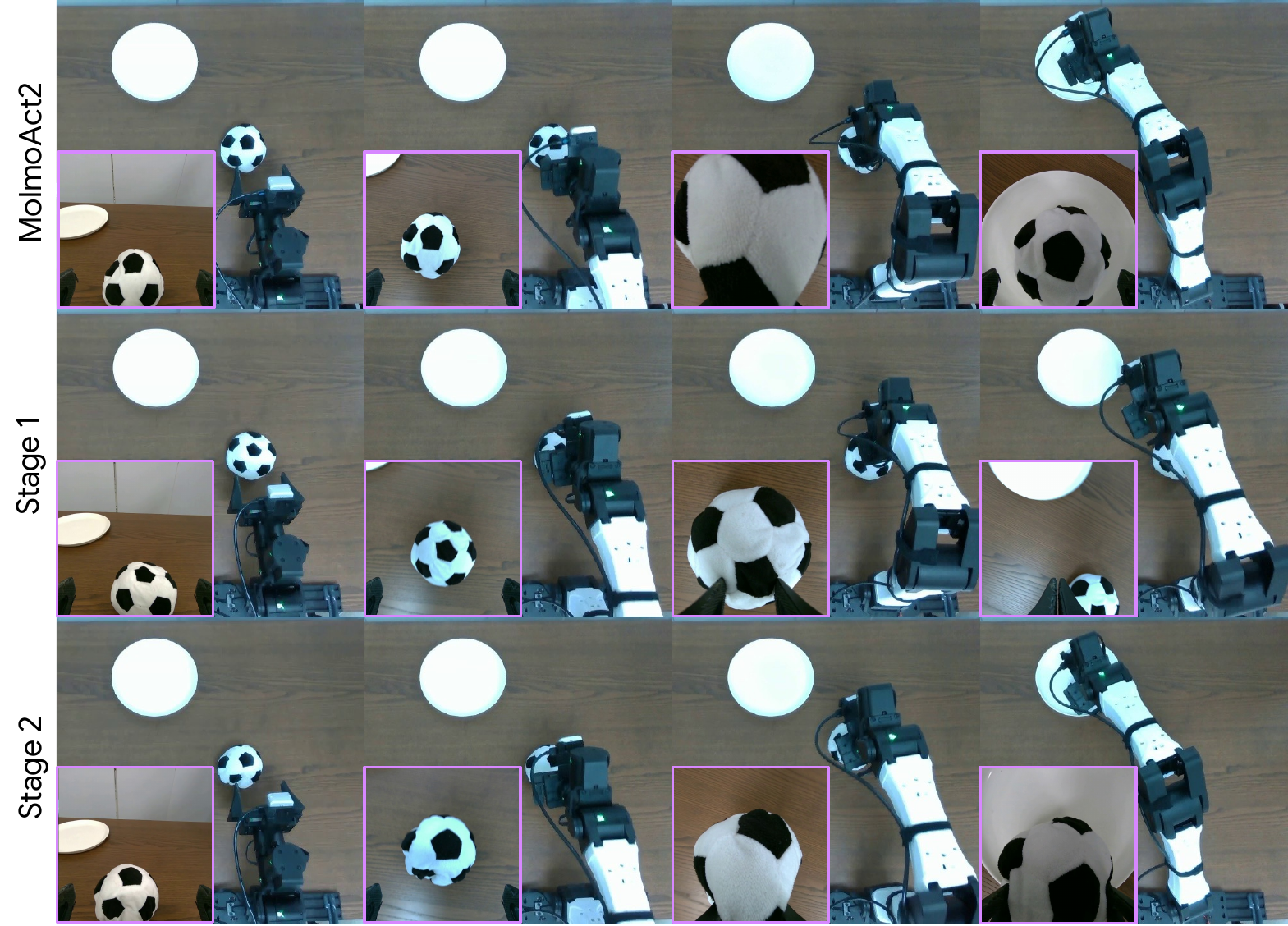}
    \caption{\textbf{Qualitative results on the real-world task.} Rollouts of the MolmoAct2 teacher and the student after Stage~1 (initialization) and Stage~2 (FastOPD) on the YAM robot arm, in which the robot picks up a ball and places it on a plate. Time progresses from left to right; insets show the wrist-camera view.}
    \label{fig:real-qual}
\end{figure}

\end{document}